\documentclass[12pt]{article}

\usepackage[margin=0.7in]{geometry}

\usepackage{amsmath,amssymb,amsthm,mathtools}
\usepackage{mathrsfs}
\numberwithin{equation}{section}

\mathtoolsset{showonlyrefs=true}

\usepackage{graphicx}
\usepackage{booktabs}
\usepackage{longtable}
\usepackage{array}
\usepackage{caption}
\usepackage{subcaption}
\usepackage{pdflscape}
\usepackage{siunitx}
\usepackage{xcolor}
\usepackage{enumitem}
\usepackage{appendix}
\usepackage{comment}

\usepackage[round]{natbib}
\usepackage[
    colorlinks=true,
    linkcolor=blue,
    citecolor=blue,
    urlcolor=blue
]{hyperref}

\usepackage{titling}
\theoremstyle{plain}
\newtheorem{theorem}{Theorem}[section]
\newtheorem{lemma}[theorem]{Lemma}
\newtheorem{proposition}[theorem]{Proposition}
\newtheorem{corollary}[theorem]{Corollary}

\theoremstyle{definition}
\newtheorem{definition}{Definition}[section]

\theoremstyle{remark}

\newcommand{\R}{\mathbb{R}}
\newcommand{\E}{\mathbb{E}}
\newcommand{\F}{\mathcal{F}}
\newcommand{\Dop}{\mathcal{D}}

\newcommand{\norm}[1]{\left\lVert#1\right\rVert}
\newcommand{\supnorm}[2]{\left\lVert#1\right\rVert_{{L^\infty([0,#2])}}}

\newcommand{\TaskOne}{\ensuremath{\mathcal{T}_1}}
\newcommand{\TaskTwo}{\ensuremath{\mathcal{T}_2}}
\newcommand{\TaskThree}{\ensuremath{\mathcal{T}_3}}

\newcommand{\RelLtwo}{\ell_{2}}
\newcommand{\RelLinf}{\ell_{\infty}}

\newcommand{\sci}[2]{\ensuremath{#1\times 10^{#2}}}
\newcommand{\bsci}[2]{\ensuremath{\boldsymbol{#1\times 10^{#2}}}}
\begin{document}
\title{Wong--Zakai Convergence and FNO Learning for Stochastic Volterra Equations}
\author{Wonjae Lee\thanks{InnoCore AI Meta-Scientist Research Group, KAIST, Daejeon 34141, South Korea. wjlee1107@kaist.ac.kr, phwjlee117@gmail.com}\quad Jiuk Jang\thanks{Department of Mathematical Sciences, Seoul National University, Seoul 08826, South Korea. jiukjang@snu.ac.kr} \quad Hyungbin Park\thanks{Department of Mathematical Sciences and Research Institute of Mathematics, Seoul National University, 1, Gwanak-ro, Gwanak-gu, Seoul 08826, South Korea. hyungbin@snu.ac.kr, hyungbin2015@gmail.com}}

\date{\today}

\maketitle

\begin{abstract} 
This paper studies Wong--Zakai convergence and Fourier neural operator learning for stochastic Volterra integral equations (SVIEs) describing random systems with memory. A Wong--Zakai theorem establishes convergence of pathwise solutions driven by piecewise linear approximations of Brownian motion to the corresponding SVIE solution. This convergence provides a foundation for using mirror-padded Fourier neural operators to learn the solution operator mapping initial conditions and driving noise to entire solution trajectories. Numerical experiments demonstrate competitive prediction accuracy and generalization to temporal grids finer than those used for training, while assessing the influence of padding and computational cost. As a follow-up to \citet{LKP:26}, this study extends their operator-learning framework to SVIEs through separate convergence and approximation analyses tailored to convolutional memory.
\end{abstract}

 \section{Introduction}

 Stochastic Volterra integral equations provide a natural framework for describing random systems with memory. Their convolutional structure allows past states and disturbances to influence the current dynamics through a memory kernel. This flexibility also creates computational challenges when many solution trajectories must be generated for different initial conditions and realizations of the driving noise. Learning the corresponding solution operator offers a way to capture this dependence directly and facilitates efficient simulation of complex stochastic dynamics.

 This paper is a follow-up to \citet{LKP:26}, who developed a Fourier neural operator framework for learning non-Markovian stochastic processes. Their work introduced the mirror-padded Fourier neural operator (MFNO) and established approximation results for path-dependent stochastic differential equations and Lipschitz transformations of fractional Brownian motion. The present study extends this research to stochastic Volterra integral equations by establishing the stochastic convergence and operator approximation results required for this setting.

Consider the stochastic Volterra integral equation
\begin{equation}
	X_t=\xi+\int_0^t\phi_0(t-s)b(X_s)\,ds
	+\int_0^t\phi(t-s)\sigma(X_s)\circ dB_s,
	\qquad 0\le t\le T,
\end{equation}
which induces a solution operator
$$
X=F(\xi,B),
$$
where $F$ maps the initial condition $\xi\in\mathbb{R}^d$ and the driving Brownian path $B\in C([0,T],\mathbb{R}^{\ell})$ to the solution path $X\in C([0,T],\mathbb{R}^d)$. Identifying $\xi$ with a constant path, we regard $F$ as an operator on paired paths:
$$
F:C([0,T],\mathbb{R}^{d+\ell})\to C([0,T],\mathbb{R}^d).
$$
The objective is to approximate $F$ using an FNO. To this end, our framework replaces $B$ by its piecewise linear interpolation and approximates $X$ through the resulting pathwise equations. Wong--Zakai convergence ensures that these pathwise solutions converge to $X$, while their deterministic solution maps provide continuous approximations of $F$ suitable for neural operator approximation. However, the preceding study did not establish the Wong--Zakai theorem required for the stochastic Volterra equations considered here, leaving them outside the scope of its approximation theory.

The present work establishes a Wong--Zakai convergence theorem for the class of convolutional stochastic Volterra equations considered here and uses it to develop Fourier neural operator approximation results.  This theorem provides the stochastic foundation for extending the operator-learning framework. Related Wong--Zakai results have been established for equations driven by fractional Brownian motion and more general Gaussian processes \citep{CoutinQian2002,FrizRiedel2014}, as well as for stochastic Volterra equations \citep{Tudor2009,XuZhang2024}. These results do not directly yield the convergence theorem required in our setting, motivating a separate analysis tailored to the kernel assumptions and approximation scheme adopted here.

 A further step is needed to obtain the neural operator approximation results. The general approximation properties of MFNOs must be combined with well-posedness and continuity of the relevant Volterra solution maps. These properties require independent proofs that account for the convolutional memory terms. By establishing them and combining them with Wong--Zakai convergence, this paper derives approximation results for the stochastic Volterra solution operators. The contribution therefore includes both the stochastic approximation theorem and the operator analysis needed to extend the preceding learning framework.

 The theoretical results are complemented by numerical experiments involving different memory kernels and drift structures. The experiments compare Fourier-based models with alternative neural architectures and evaluate their performance on temporal grids finer than those used for training. They also examine the influence of padding on prediction accuracy. Mirror padding provides the structure used in the theoretical approximation analysis, while the empirical comparisons assess the performance of the different architectures and padding strategies.

This paper is a follow-up to \citet{LKP:26}, which introduced an MFNO framework for non-Markovian stochastic processes, established rigorous approximation theorems, and demonstrated prediction accuracy and resolution generalization through numerical experiments. The present study extends these contributions to stochastic Volterra integral equations as follows.

First, the operator-learning framework is extended to stochastic dynamics with convolutional memory. Using the MFNO architecture developed in the preceding study, the proposed approach learns the mapping from initial conditions and driving Brownian paths to entire solution trajectories. This broadens the scope of the framework, previously applied to path-dependent stochastic differential equations and Lipschitz transformations of fractional Brownian motion, to the Volterra equations considered here.

Second, this paper establishes the theoretical results required for this extension. The preceding study did not prove the Wong--Zakai theorem needed for these Volterra equations and therefore did not include them in its approximation theory. Our analysis establishes the corresponding convergence result, proves well-posedness and continuity of the associated deterministic solution maps, and derives Fourier neural operator approximation theorems. The convolutional memory structure requires separate arguments, so these conclusions do not follow directly from the earlier results.

Third, numerical experiments show that the prediction accuracy and resolution generalization demonstrated in the preceding study also extend to the Volterra examples considered here. Comparisons with alternative neural architectures demonstrate competitive accuracy, while evaluations on finer temporal grids show that trained Fourier-based models retain comparable performance across resolutions. The experiments also assess the influence of padding and the computational cost of evaluating the learned solution operators.

The present study is related to two strands of literature: pathwise approximation of stochastic equations and neural operator learning. The classical work of \citet{WongZakai1965} established the connection between ordinary equations driven by smooth Brownian approximations and stochastic equations with a corrected drift. Rough path extensions include the construction of fractional Brownian rough paths by \citet{CoutinQian2002} and quantitative approximation results for Gaussian rough paths by \citet{FrizRiedel2014}. For equations with memory, \citet{Tudor2009} proved quadratic mean convergence, uniformly in time, for Brownian piecewise linear approximations of scalar It\^o--Volterra equations associated with the fractional Brownian kernel, with a modified drift in the limiting equation. \citet{XuZhang2024} established Wong--Zakai approximation results for stochastic Volterra equations, including applications to fractional and subfractional Brownian kernels with Hurst parameter greater than $1/2$. A complementary pathwise approach was developed by \citet{DeyaTindel,secondroughvol} through algebraic and convolutional rough integration, while \citet{volterraroughsignal} introduced Volterra signatures and controlled paths to treat possibly singular kernels.

On the operator-learning side, \citet{Li:21} introduced FNOs for parametric PDEs, and \citet{Kovachki:21} established universal approximation results for continuous operators between Sobolev spaces on compact input sets. \citet{Salvi2022} developed neural SPDEs that learn solution operators depending jointly on initial conditions and driving-noise realizations through the mild formulation. More directly related, \citet{LKP:26} introduced MFNOs and combined Wong--Zakai approximation with operator universality for path-dependent SDEs and Lipschitz transformations of fractional Brownian motion. Building on these developments, the present work connects a direct Wong--Zakai analysis for Brownian-driven convolutional SVIEs, under an absolutely continuous diffusion kernel with derivative of bounded variation, with well-posedness and continuity of deterministic Volterra solution maps to establish MFNO approximation results for the corresponding stochastic solution operators.

The remainder of the paper is organized as follows. Section~\ref{WZ} establishes the Wong--Zakai convergence theorem for stochastic Volterra equations. Section~\ref{sec:FNO} reviews Fourier neural operators and their universal approximation properties. Section~\ref{sec:approx} develops the approximation theory for Volterra solution maps. Section~\ref{sec:experi} describes the experimental setup, and Section~\ref{sec:results} presents the numerical results on prediction accuracy, resolution generalization, and computational performance. Section \ref{sec:con} summarizes the paper.

\section{Wong--Zakai Approximation}
\label{WZ}
\subsection{Stochastic Volterra Equations}

Let $(\Omega,\mathcal{F},(\mathcal{F}_t)_{t\ge0},\mathbb{P})$ be a filtered probability space satisfying the usual conditions and supporting a standard $m$-dimensional Brownian motion $B=(B^1,\ldots,B^m)$ with respect to $(\mathcal{F}_t)_{t\ge0}$. Throughout this paper, we impose the following assumptions:
	\begin{equation}\label{eq:assumptions}
		\begin{gathered}
		  T>0,\qquad	\phi_0\in W^{1,1}(0,T;\mathbb R),\\
			\phi\in AC([0,T];\mathbb R^{d\times d}),\qquad
			q:=\phi'\in BV(0,T;\mathbb R^{d\times d}),\\
			b:\mathbb R^d\to\mathbb R^d\ \text{is globally Lipschitz},
			\qquad \sigma\in C_b^2(\mathbb R^d;\mathbb R^{d\times m}).
		\end{gathered}
	\end{equation}
We write
\[
\sigma(y):=\bigl(\sigma_1(y),\ldots,\sigma_m(y)\bigr),
\]
that is, $\sigma_j(y)=\sigma(y)e_j$, where $e_j$ is the $j$th standard basis
vector of $\mathbb R^m$.

Here $AC$ means entrywise absolute continuity. The space
$BV(0,T;\mathbb R^{d\times d})$ consists of integrable matrix-valued
functions whose distributional derivative is a finite matrix-valued
measure.
We use $|\cdot|$ for the Euclidean norm of vectors and the
Frobenius norm of matrices. For a matrix-valued measure $\mu$, its
variation measure $|\mu|$ is defined using the Frobenius norm.
An integral $\int\mu(dr)\,v_r$ is interpreted entrywise, with the
matrix-valued measure multiplying the vector on the left.
We use the unique absolutely continuous representative of $\phi_0$ on
$[0,T]$. The derivative $q$ is understood almost everywhere; its
one-sided traces exist and are finite. In particular,
$q\in L^\infty(0,T;\mathbb R^{d\times d})$ and $\phi$ is Lipschitz.
No classical second derivative of $\phi$ is assumed.
The space $C_b^2$ consists of functions with continuous bounded
derivatives through order two, including the function itself.
No boundedness or differentiability assumption is imposed on $b$.

We study the Stratonovich Volterra equation
\begin{equation}\label{eq:stratonovich}
	Y_t={y_0}+\int_0^t\phi_0(t-s)b(Y_s)\,ds
	+\int_0^t\phi(t-s)\sigma(Y_s)\circ dB_s,
	\qquad 0\le t\le T
\end{equation}
for a deterministic initial value $y_0\in\mathbb R^d$.
We denote the solution with initial value $y_0$ by ${Y^{y_0}}$ and write $Y=Y^{y_0}$ for brevity.
The circle denotes Stratonovich integration. For each time
$t$, the stochastic integral in \eqref{eq:stratonovich} is the usual
Stratonovich integral in the variable $s$.

{Let $N_n=2^n$ and define the dyadic partition
\[
\pi_n:=\{t_k^n:0\le k\le N_n\},
\qquad
t_k^n:=kh_n,
\qquad
h_n:=|\pi_n|=T/N_n.
\]
}
Define the piecewise linear interpolation of $B$ by
\begin{equation}\label{eq:interpolation}
	B_t^n=B_{{t_k^n}}+\frac{t-{t_k^n}}{h_n}(B_{{t_{k+1}^n}}-B_{{t_k^n}}),
	\qquad t\in[{t_k^n},{t_{k+1}^n}].
\end{equation}
The approximating equation is
\begin{equation}\label{eq:approximation}
Y_t^n={y_0}+\int_0^t\phi_0(t-s)b(Y_s^n)\,ds
+\int_0^t\phi(t-s)\sigma(Y_s^n)\,dB_s^n\,,\qquad 0\le t\le T\,.
\end{equation}
We denote the solution with initial value $y_0$ by ${Y^{n,y_0}}$ and write $Y^n=Y^{n,y_0}$ for brevity.
The last integral is an ordinary Riemann--Stieltjes integral.

\begin{theorem}[Wong--Zakai convergence]\label{thm:main}
	Under the assumptions in \eqref{eq:assumptions}, the following statements hold:
	\begin{enumerate}
		\item There exists a unique continuous adapted semimartingale solution $Y$ to \eqref{eq:stratonovich} such that $$\E\supnorm{Y}{T}^2<\infty\,.$$
		\item For each $n$, there exists a unique continuous pathwise solution $Y^n$ to \eqref{eq:approximation}.
		\item There exists a constant $C<\infty$ such that, for all sufficiently large $n$,
		\begin{equation}\label{eq:convergence}
			\E\!\left[\sup_{0\le t\le T}|Y_t^n-Y_t|^2\right]
			\le C h_n\log\!\left(\frac{2T}{h_n}\right).
		\end{equation}
		Moreover,
		\begin{equation}\label{eq:asconvergence}
			\sup_{0\le t\le T}|Y_t^n-Y_t|\longrightarrow0
			\quad\text{almost surely as }n\to\infty.
		\end{equation}
	\end{enumerate}
\end{theorem}

The proof below is direct and does not invoke a Wong--Zakai theorem for an
It\^o equation. It\^o integrals appear only as a tool for constructing the
Stratonovich solution and for applying martingale estimates.

\subsection{Convolution Identities and Well-Posedness}
{
Set
\begin{equation}\label{eq:coefficients}
\begin{gathered}
G_j(y):=D\sigma_j(y)\,\phi(0)\sigma_j(y),\qquad \rho_j(y):=\frac12G_j(y),\qquad
\rho(y):=\sum_{j=1}^m\rho_j(y).
\end{gathered}
\end{equation}
The assumption $\sigma\in C_b^2$ ensures that each $\rho_j$ is bounded
and globally Lipschitz. Indeed,
\[
\|D\rho_j\|_\infty
\le \frac{|\phi(0)|}{2}
\left(\|D^2\sigma_j\|_\infty\|\sigma_j\|_\infty
+\|D\sigma_j\|_\infty^2\right).
\]
The order $D\sigma_j(y)\phi(0)\sigma_j(y)$ is essential; no commutativity
of $\phi(0)$ and $D\sigma_j(y)$ is assumed.
For continuous paths $y$ and $z$, define
\begin{equation}\label{eq:operators}
\begin{aligned}
(\Dop y)_t&:=\int_0^t\phi_0(t-s)b(y_s)\,ds,\\
(\mathcal Kz)_t&:=\phi(0) z_t+\int_0^tq(t-s)z_s\,ds.
\end{aligned}
\end{equation}
If $z_0=0$ and $z$ is a continuous semimartingale or a continuous path
of bounded variation, integration by parts gives
\begin{equation}\label{eq:ibp}
\int_0^t\phi(t-s)\,dz_s=(\mathcal Kz)_t.
\end{equation}
The identity holds for both It\^o and Stratonovich integration because
the deterministic matrix-valued kernel has finite variation. Moreover,
\begin{equation}\label{eq:Lbound}
\supnorm{\mathcal Kz}{t}\le A_\phi\supnorm{z}{t},
\qquad A_\phi:=|\phi(0)|+\int_0^T|q(r)|\,dr.
\end{equation}
All constants denoted by $C$ below may depend on $T,{y_0},d,m$,
$|b(0)|$, the global Lipschitz constant of $b$,
$\|\sigma\|_\infty$, $\|D\sigma\|_\infty$, $\|D^2\sigma\|_\infty$,
and the kernel quantities
\[
|\phi_0(0)|+\|\phi_0'\|_{L^1},\qquad
|\phi(0)|+\|q\|_{L^1}+|q(0+)|+|Dq|((0,T)),
\]
but not on $n$ or a mesh index.

}

{
\begin{lemma}\label{lem:BV}
Let $q\in BV(0,T;\mathbb R^{d\times d})$,  $\mu=Dq$ on $(0,T)$,
and $z\in C([0,T];\mathbb R^d)$ satisfy $z_0=0$. Define
\[
U_t:=\int_0^tq(t-s)z_s\,ds.
\]
Then $U\in C^1([0,T];\mathbb R^d)$, with one-sided derivatives at the
endpoints. If $\bar z_s=z_s$ for $s\in[0,T]$ and $\bar z_s=0$ for $s<0$,
then
\begin{equation}\label{eq:BVderivative}
U_t'=q(0+)z_t+\int_{(0,T)}\mu(dr)\,\bar z_{t-r},
\qquad 0\le t\le T.
\end{equation}
In particular,
\begin{equation}\label{eq:BVbound}
|U_t'|\le\bigl(|q(0+)|+|\mu|((0,T))\bigr)\supnorm{z}{t}.
\end{equation}
\end{lemma}

\begin{proof}
Choose a right-continuous representative on $(0,T)$, so that
$q(r)=q(0+)+\mu((0,r])$ for almost every $r\in(0,T)$.
Applying Fubini's theorem entrywise gives
\begin{equation}\label{eq:BVfubini}
U_t=q(0+)\int_0^tz_s\,ds+
\int_{(0,T)}\mu(dr)\left(\int_0^{(t-r)_+}z_s\,ds\right).
\end{equation}
For a continuous vector-valued integrand $v$, the bound
$|\int\mu(dr)v_r|\le\int|v_r|\,|\mu|(dr)$ follows from the
matrix norm inequality $|Mv|\le|M|\,|v|$.
Since $z_0=0$, the function $t\mapsto\int_0^{(t-r)_+}z_s\,ds$
is continuously differentiable, with derivative $\bar z_{t-r}$.
This derivative is bounded by $\supnorm{z}{T}$ uniformly in $r$.
Differentiation under the finite variation measure gives
\eqref{eq:BVderivative}. Its right-hand side is continuous by dominated
convergence, including at the endpoints, because the extension $\bar z$
is continuous at zero. Only $r\le t$ contribute, which also proves
\eqref{eq:BVbound}. The derivative measure is taken on the open interval
$(0,T)$, without artificial boundary atoms.
\end{proof}

}

\begin{lemma}\label{lem:deterministic}
Let continuous paths $y,z^1,\ldots,z^m$, with $z_0^j=0$ for all $j=1,\ldots,m$, satisfy
\begin{equation}\label{eq:representation}
	y=y_0+\Dop y+\sum_{j=1}^m\mathcal K z^j.
\end{equation}
Then, for $t\le T$,
\begin{equation}\label{eq:growth}
	\supnorm{y}{t}\le C\left(1+\sum_{j=1}^m\supnorm{z^j}{t}\right).
\end{equation}
If $\widehat y,\widehat z^1,\ldots,\widehat z^m$ satisfy the same conditions
with $\widehat y_0=y_0$, then
\begin{equation}\label{eq:stability}
	\supnorm{y-\widehat y}{t}
	\le C\sum_{j=1}^m\supnorm{z^j-\widehat z^j}{t}.
\end{equation}
Moreover, defining
\begin{equation}\label{eq:Vdefinition}
	V_t:=(\Dop y)_t+\sum_{j=1}^m\int_0^t q(t-s)z_s^j\,ds,
\end{equation}
we have $y_t=y_0+\phi(0)\sum_j z_t^j+V_t$, $V\in C^1([0,T];\R^d)$, and
\begin{equation}\label{eq:Vbound}
	|V_t'|\le C\left(1+\sum_{j=1}^m\supnorm{z^j}{t}\right).
\end{equation}
\end{lemma}

\begin{proof}
Let $L_b$ be a global Lipschitz constant of $b$, and set
$M_0:=\|\phi_0\|_{L^\infty(0,T)}$. Since $\phi_0\in W^{1,1}(0,T)$,
\[
M_0\le |\phi_0(0)|+\|\phi_0'\|_{L^1(0,T)}<\infty.
\]
Also, $|b(x)|\le |b(0)|+L_b|x|$.
By \eqref{eq:representation} and \eqref{eq:Lbound}, for $0\le u\le t\le T$,
\[
\supnorm{y}{u}
\le |y_0|+M_0T|b(0)|
+A_\phi\sum_{j=1}^m\supnorm{z^j}{t}
+M_0L_b\int_0^u\supnorm{y}{s}\,ds.
\]
Fixing $t$ and applying Gronwall's inequality on $[0,t]$ gives
\[
\begin{aligned}
\supnorm{y}{t}
&\le
\left(
|y_0|+M_0T|b(0)|
+A_\phi\sum_{j=1}^m\supnorm{z^j}{t}
\right)e^{M_0L_bt}\\
&\le C\left(1+\sum_{j=1}^m\supnorm{z^j}{t}\right).
\end{aligned}
\]

For the stability estimate, subtract the two representations:
\[
y-\widehat y
=\Dop y-\Dop\widehat y
+\sum_{j=1}^m\mathcal K(z^j-\widehat z^j).
\]
The Lipschitz continuity of $b$ implies, for $0\le u\le t$,
\[
\supnorm{y-\widehat y}{u}
\le M_0L_b\int_0^u\supnorm{y-\widehat y}{s}\,ds
+A_\phi\sum_{j=1}^m\supnorm{z^j-\widehat z^j}{t}.
\]
Hence Gronwall's inequality yields
\[
\supnorm{y-\widehat y}{t}
\le A_\phi e^{M_0L_bt}
\sum_{j=1}^m\supnorm{z^j-\widehat z^j}{t}
\le C\sum_{j=1}^m\supnorm{z^j-\widehat z^j}{t}.
\]

It remains to establish the regularity and derivative bound for $V$.
Its definition immediately gives
\[
y_t=y_0+\phi(0)\sum_{j=1}^m z_t^j+V_t.
\]
Set $f(s):=b(y_s)$, which is continuous and bounded on $[0,T]$.
Absolute continuity of $\phi_0$ and Fubini's theorem give
\[
\begin{aligned}
(\Dop y)_t
&=\int_0^t
\left(\phi_0(0)+\int_s^t\phi_0'(u-s)\,du\right)f(s)\,ds\\
&=\phi_0(0)\int_0^t f(s)\,ds
+\int_0^t\left(\int_0^u\phi_0'(u-s)f(s)\,ds\right)du.
\end{aligned}
\]
To verify continuity of the inner convolution, extend
$g:=\phi_0'$ and $f$ by zero to $\mathbb R$ and write $H:=g*f$.
Then
\[
\|H(\cdot+h)-H\|_{L^\infty(\mathbb R)}
\le
\|g(\cdot+h)-g\|_{L^1(\mathbb R)}
\|f\|_{L^\infty(\mathbb R)}
\longrightarrow0
\]
as $h\to0$, by continuity of translations in $L^1(\mathbb R)$.
Thus $\Dop y\in C^1([0,T];\R^d)$, with
\[
(\Dop y)'_t
=\phi_0(0)b(y_t)
+\int_0^t\phi_0'(t-s)b(y_s)\,ds.
\]

Let $\mu:=Dq$ on $(0,T)$ and extend each $z^j$ by zero to negative
times, denoting the extension by $\bar z^j$. Since $z_0^j=0$,
this extension is continuous. Applying Lemma~\ref{lem:BV} to each
noise memory term gives $V\in C^1([0,T];\R^d)$ and
\begin{align}\label{eq:Vderivative}
V_t'
={}&\phi_0(0)b(y_t)
+\int_0^t\phi_0'(t-s)b(y_s)\,ds\notag\\
&+\sum_{j=1}^m
\left[
q(0+)z_t^j
+\int_{(0,T)}\mu(dr)\,\bar z_{t-r}^j
\right].
\end{align}
The endpoint derivatives are understood one-sided. Moreover,
\[
\left|\int_{(0,T)}\mu(dr)\,\bar z_{t-r}^j\right|
\le |\mu|((0,T))\supnorm{z^j}{t}.
\]
Combining this estimate with the linear growth of $b$ and the
growth bound for $y$ proved above, we obtain
\[
\begin{aligned}
|V_t'|
&\le
\left(|\phi_0(0)|+\|\phi_0'\|_{L^1(0,T)}\right)
\left(|b(0)|+L_b\supnorm{y}{t}\right)\\
&\quad+
\left(|q(0+)|+|\mu|((0,T))\right)
\sum_{j=1}^m\supnorm{z^j}{t}\\
&\le C\left(1+\sum_{j=1}^m\supnorm{z^j}{t}\right).
\end{aligned}
\]
\end{proof}

\begin{proposition}
	\label{prop:wellposedness}
The Stratonovich equation \eqref{eq:stratonovich} is equivalent,
within the class of continuous adapted semimartingales $Y$ satisfying
$\E\supnorm{Y}{T}^2<\infty$, to 
\begin{equation}\label{eq:itoform}
\begin{aligned}
Y_t={}&{y_0}+\int_0^t\phi_0(t-s)b(Y_s)\,ds
+\int_0^t\phi(t-s)\rho(Y_s)\,ds\\
&+\int_0^t\phi(t-s)\sigma(Y_s)\,dB_s.
\end{aligned}
\end{equation}
This equation has a unique continuous adapted solution with a finite
	second moment of its supremum, and that solution is a semimartingale.
\end{proposition}

\begin{proof}
	For a square-integrable predictable $\R^d$-valued process $H$, let
	$M_t=\int_0^tH_s\,dB_s^j$. From \eqref{eq:ibp}, \eqref{eq:Lbound},
	Doob's inequality, and the It\^o isometry,
	{
\begin{equation}\label{eq:convolutionestimate}
\E\sup_{u\le t}\left|\int_0^u\phi(u-s)H_s\,dB_s^j\right|^2
\le4A_\phi^2\int_0^t\E|H_s|^2\,ds.
\end{equation}
}
	In particular, the convolution itself need not be a martingale.
	
	Let $\Psi$ denote the right-hand side of \eqref{eq:itoform}. Ordinary
	Cauchy--Schwarz estimates for the drift convolutions and
	\eqref{eq:convolutionestimate} yield
	\[
	\E\supnorm{\Psi U-\Psi W}{t}^2
	\le C\int_0^t\E\supnorm{U-W}{s}^2\,ds.
	\]
	Starting from $Y^{(0)}\equiv {y_0}$ and setting $Y^{(r+1)}=\Psi Y^{(r)}$,
	the squared norms of successive differences are bounded at $T$ by
	$C_0(CT)^r/r!$. The square roots of these bounds are summable.
	Thus the Picard sequence converges in
	$L^2(\Omega;C([0,T];\R^d))$ to a solution. The same estimate and
	Gronwall's inequality give uniqueness. Linear growth and
	\eqref{eq:convolutionestimate} give the stated moment bound.
	
	Define
	\begin{equation}\label{eq:Zlimit}
		Z_t^j:=\int_0^t\sigma_j(Y_s)\,dB_s^j
		+\int_0^t{\rho_j}(Y_s)\,ds.
	\end{equation}
	By integration by parts, $Y={y_0}+\Dop Y+\sum_j{\mathcal K}Z^j$.
	Lemma~\ref{lem:deterministic} therefore decomposes $Y$ into
	${y_0}+{\phi(0)\sum_jZ^j}$ and a continuously differentiable process. Consequently,
	$Y$ is a semimartingale and
	{
\begin{equation}\label{eq:covariation}
d[Y,B^j]_s=\phi(0)\sigma_j(Y_s)\,ds.
\end{equation}
}
	Since $\sigma_j\in C^2$, It\^o's formula shows that
	$\sigma_j(Y)$ is a continuous semimartingale. The Stratonovich
	conversion formula therefore gives
	{
\begin{equation}\label{eq:conversion}
\begin{aligned}
\int_0^t\phi(t-s)\sigma_j(Y_s)\circ dB_s^j
={}&\int_0^t\phi(t-s)\sigma_j(Y_s)\,dB_s^j\\
&+\frac12\int_0^t\phi(t-s)D\sigma_j(Y_s)
\phi(0)\sigma_j(Y_s)\,ds\\
={}&\int_0^t\phi(t-s)\sigma_j(Y_s)\,dB_s^j
+\int_0^t\phi(t-s)\rho_j(Y_s)\,ds.
\end{aligned}
\end{equation}
}
	Thus the constructed solution solves \eqref{eq:stratonovich}.
	Conversely, for a continuous adapted semimartingale solution of
	\eqref{eq:stratonovich}, set
	$Z^j=\int_0^\cdot\sigma_j(Y_s)\circ dB_s^j$.
	The same convolution identity and finite-variation decomposition give
	\eqref{eq:covariation} and then \eqref{eq:conversion}.
	Hence it solves \eqref{eq:itoform}, proving the equivalence and uniqueness.
\end{proof}

For fixed $n$ and a fixed Brownian sample path, $B^n$ is absolutely
continuous and has finite variation. Standard deterministic Picard
iteration gives a unique solution of \eqref{eq:approximation} on $[0,T]$:
the integral Lipschitz bound has the integrable coefficient
$C(1+|\dot B_s^n|)$, and the same bound with linear growth prevents
explosion. More explicitly, iteration gives factorial bounds with
$\int_0^T C(1+|\dot B_s^n|)\,ds$ in place of $CT$.
Define the uncorrected auxiliary processes
\begin{equation}\label{eq:Zapprox}
	Z_t^{n,j}:=\int_0^t\sigma_j(Y_s^n)\,dB_s^{n,j}.
\end{equation}
Then
\begin{equation}\label{eq:Yapproxrepresentation}
	Y^n={y_0}+\Dop Y^n+\sum_j{\mathcal K}Z^{n,j}
	={y_0}+{\phi(0)\sum_jZ^{n,j}}+V^n,
\end{equation}
where $V^n$ is given by \eqref{eq:Vdefinition}.

\subsection{A Direct Expansion on Each Mesh Interval}

Fix $n$ and write $h=h_n$. Then set
\[
\Delta_kB:=B_{{t_{k+1}^n}}-B_{{t_k^n}},\qquad
R_k:=|\Delta_kB|,\qquad y_k:=Y_{{t_k^n}}^n,
\qquad
S_k:=\sum_{j=1}^m\supnorm{Z^{n,j}}{{t_k^n}}.
\]
The entire paths $Y^n$ and $Z^{n,j}$ up to ${t_k^n}$ depend only on
$B$ up to ${t_k^n}$. Hence $y_k$ and $S_k$ are $\F_{{t_k^n}}$-measurable,
and $S_k$ is independent of $\Delta_kB$. At times inside a mesh
interval the interpolated path, and generally its solution, need not
be adapted.

\begin{lemma}\label{lem:mesh}
	For $h\le1$ and each $k=0,\ldots,{N_n}-1$,
	{
\begin{equation}\label{eq:meshexpansion}
\begin{aligned}
Z_{t_{k+1}^n}^{n,j}-Z_{t_k^n}^{n,j}
={}&\sigma_j(y_k)\Delta_kB^j+\rho_j(y_k)h\\
&+\frac12\sum_{\ell=1}^m
D\sigma_j(y_k)\phi(0)\sigma_\ell(y_k)
\bigl(\Delta_kB^j\Delta_kB^\ell-\delta_{j\ell}h\bigr)
+r_k^j.
\end{aligned}
\end{equation}
}
	where $\delta_{j\ell}$ is the Kronecker delta and
	\begin{equation}\label{eq:remainder}
		\sum_{j=1}^m|r_k^j|
		\le C\left(R_k^3+hR_k(1+R_k)(1+S_k)\right).
	\end{equation}
\end{lemma}

\begin{proof}
Fix $k\in\{0,\ldots,N_n-1\}$. On $[t_k^n,t_{k+1}^n]$, we have
$dB_s^{n,j}=h^{-1}\Delta_kB^j\,ds$. Hence boundedness of $\sigma_j$
gives
\[
|Z_s^{n,j}-Z_{t_k^n}^{n,j}|
\le \|\sigma_j\|_\infty
\frac{s-t_k^n}{h}|\Delta_kB^j|,
\]
and therefore
\begin{equation}\label{eq:localZ}
\sum_{j=1}^m\sup_{t_k^n\le s\le t_{k+1}^n}
|Z_s^{n,j}-Z_{t_k^n}^{n,j}|
\le CR_k.
\end{equation}
In particular,
\[
\sum_{j=1}^m\supnorm{Z^{n,j}}{t_{k+1}^n}
\le S_k+CR_k.
\]
Applying \eqref{eq:Vbound} and integrating over the mesh interval yields
\begin{equation}\label{eq:localV}
v_k:=\sup_{t_k^n\le s\le t_{k+1}^n}|V_s^n-V_{t_k^n}^n|
\le Ch(1+S_k+R_k).
\end{equation}
By \eqref{eq:Yapproxrepresentation},
\[
Y_s^n-y_k
=\phi(0)\sum_{\ell=1}^m
\bigl(Z_s^{n,\ell}-Z_{t_k^n}^{n,\ell}\bigr)
+V_s^n-V_{t_k^n}^n.
\]
Consequently,
\begin{equation}\label{eq:localY}
\sup_{t_k^n\le s\le t_{k+1}^n}|Y_s^n-y_k|
\le C(R_k+v_k).
\end{equation}

For $s\in[t_k^n,t_{k+1}^n]$, define
\[
\theta_s:=\frac{s-t_k^n}{h},
\qquad
w_s:=\theta_s\sum_{\ell=1}^m
\phi(0)\sigma_\ell(y_k)\Delta_kB^\ell,
\qquad
e_s:=Y_s^n-y_k-w_s.
\]
Subtracting $w_s$ from the preceding representation gives
\[
e_s
=\phi(0)\sum_{\ell=1}^m
\int_{t_k^n}^s
\bigl(\sigma_\ell(Y_u^n)-\sigma_\ell(y_k)\bigr)\,dB_u^{n,\ell}
+V_s^n-V_{t_k^n}^n.
\]
Using the Lipschitz continuity of $\sigma_\ell$, we obtain
\begin{equation}\label{eq:errorleading}
\begin{aligned}
\sup_{t_k^n\le s\le t_{k+1}^n}|e_s|
&\le CR_k\sup_{t_k^n\le u\le t_{k+1}^n}|Y_u^n-y_k|+v_k\\
&\le C\bigl[R_k^2+(1+R_k)v_k\bigr].
\end{aligned}
\end{equation}

To estimate the coefficient remainder without introducing $v_k^2$,
write
\[
\begin{aligned}
&\sigma_j(Y_s^n)-\sigma_j(y_k)-D\sigma_j(y_k)w_s\\
&\quad=
\sigma_j(y_k+w_s+e_s)-\sigma_j(y_k+w_s)\\
&\qquad+
\sigma_j(y_k+w_s)-\sigma_j(y_k)-D\sigma_j(y_k)w_s.
\end{aligned}
\]
The mean value theorem and Taylor's formula imply
\[
\begin{aligned}
&|\sigma_j(Y_s^n)-\sigma_j(y_k)-D\sigma_j(y_k)w_s|\\
&\quad\le
\|D\sigma_j\|_\infty|e_s|
+\frac12\|D^2\sigma_j\|_\infty|w_s|^2.
\end{aligned}
\]
Since $|w_s|\le CR_k$, it follows that
\[
\sup_{t_k^n\le s\le t_{k+1}^n}
|\sigma_j(Y_s^n)-\sigma_j(y_k)-D\sigma_j(y_k)w_s|
\le C\bigl[R_k^2+(1+R_k)v_k\bigr].
\]

Define
\[
r_k^j:=
\int_{t_k^n}^{t_{k+1}^n}
\bigl[\sigma_j(Y_s^n)-\sigma_j(y_k)-D\sigma_j(y_k)w_s\bigr]
\,dB_s^{n,j}.
\]
Using
\[
\int_{t_k^n}^{t_{k+1}^n}\theta_s\,dB_s^{n,j}
=\frac12\Delta_kB^j,
\]
we obtain
\begin{equation}\label{eq:rawexpansion}
\begin{aligned}
Z_{t_{k+1}^n}^{n,j}-Z_{t_k^n}^{n,j}
={}&\sigma_j(y_k)\Delta_kB^j\\
&+\frac12\sum_{\ell=1}^m
D\sigma_j(y_k)\phi(0)\sigma_\ell(y_k)
\Delta_kB^\ell\Delta_kB^j+r_k^j.
\end{aligned}
\end{equation}
The total variation of $B^{n,j}$ on this interval is
$|\Delta_kB^j|$. Thus
\[
\begin{aligned}
\sum_{j=1}^m|r_k^j|
&\le CR_k\bigl[R_k^2+(1+R_k)v_k\bigr]\\
&\le C\bigl[R_k^3+hR_k(1+R_k)(1+S_k+R_k)\bigr]\\
&\le C\bigl[R_k^3+hR_k(1+R_k)(1+S_k)\bigr].
\end{aligned}
\]
The last inequality uses $h\le1$ to absorb $hR_k^3$ into $R_k^3$,
and absorbs $hR_k^2$ into $hR_k(1+R_k)(1+S_k)$.

Finally, since $y_k$ is $\F_{t_k^n}$-measurable and
\[
\E[\Delta_kB^\ell\Delta_kB^j\mid\F_{t_k^n}]
=\delta_{j\ell}h,
\]
the conditional mean of the quadratic term is
\[
\frac12D\sigma_j(y_k)\phi(0)\sigma_j(y_k)h
=\rho_j(y_k)h.
\]
Centering the quadratic products therefore gives
\[
\begin{aligned}
Z_{t_{k+1}^n}^{n,j}-Z_{t_k^n}^{n,j}
={}&\sigma_j(y_k)\Delta_kB^j+\rho_j(y_k)h\\
&+\frac12\sum_{\ell=1}^m
D\sigma_j(y_k)\phi(0)\sigma_\ell(y_k)
\bigl(\Delta_kB^j\Delta_kB^\ell-\delta_{j\ell}h\bigr)
+r_k^j,
\end{aligned}
\]
which completes the proof.
\end{proof}

\subsection{Uniform Moments and Comparison with Adapted Integrals}

We first control the accumulated remainders. For every $p>0$,
\begin{equation}\label{eq:Gaussianmoments}
	\E R_k^p=C_{p,m}h^{p/2}.
\end{equation}
For each fixed $n$, boundedness of $\sigma_j$ gives the pathwise bound
\begin{equation}\label{eq:fixednmoments}
	S_{{N_n}}\le C\sum_{k=0}^{{N_n}-1}R_k.
\end{equation}
Thus $S_{{N_n}}$ has a finite second moment before any uniform estimate is
established. Denote $A_n:=\norm{S_{{N_n}}}_{{L^2(\Omega)}}$.
By \eqref{eq:remainder}, independence of $S_k$ and $\Delta_kB$, and
Minkowski's inequality,
\begin{equation}\label{eq:remainderL2}
	\begin{aligned}
		\left\lVert\sum_{k=0}^{{N_n}-1}\sum_j|r_k^j|\right\rVert_{{L^2(\Omega)}}
		&\le C\sum_{k=0}^{{N_n}-1}
		\left[h^{3/2}+h\norm{R_k(1+R_k)}_{{L^2(\Omega)}}
		\norm{1+S_k}_{{L^2(\Omega)}}\right]\\
		&\le C\sqrt h\,(1+A_n).
	\end{aligned}
\end{equation}
Here and below estimates involving $\sqrt h$ are for $0<h\le1$.

For $i=0,\ldots,{N_n}$, define
\begin{align*}
	M_i^j&:=\sum_{k<i}\sigma_j(y_k)\Delta_kB^j,\\
	Q_i^j&:=\frac12\sum_{k<i}\sum_{\ell=1}^m
	{D\sigma_j(y_k)\phi(0)\sigma_\ell(y_k)}
	(\Delta_kB^j\Delta_kB^\ell-\delta_{j\ell}h),\\
	A_i^j&:=\sum_{k<i}{\rho_j}(y_k)h.
\end{align*}
Both $M^j$ and $Q^j$ are discrete-time martingales. Their coefficients
are uniformly bounded, and the conditional second moments of their
increments are bounded by $Ch$ and $Ch^2$, respectively. Consequently,
Doob's inequality gives
\begin{equation}\label{eq:martingalebounds}
	\sum_j\E\max_{i\le {N_n}}|M_i^j|^2\le C,
	\qquad
	\sum_j\E\max_{i\le {N_n}}|Q_i^j|^2\le Ch.
\end{equation}
Also $\max_i|A_i^j|\le T\norm{{\rho_j}}_\infty$.
Summing \eqref{eq:meshexpansion} gives
\begin{equation}\label{eq:gridrepresentation}
	Z_{{t_i^n}}^{n,j}=M_i^j+A_i^j+Q_i^j+\sum_{k<i}r_k^j.
\end{equation}
Together with \eqref{eq:localZ}, this implies
\[
A_n\le C+C\left\lVert\max_{k<{N_n}}R_k\right\rVert_{{L^2(\Omega)}}
+C\sqrt h\,(1+A_n).
\]
Since $\max_kR_k\le2\supnorm{B}{T}$ and
$\E\supnorm{B}{T}^2<\infty$, the middle term is uniformly bounded.
For sufficiently small $h$, the term $C\sqrt h\,A_n$ can be absorbed
into the left-hand side. The finitely many remaining meshes are covered
by \eqref{eq:fixednmoments}. Hence
\begin{equation}\label{eq:uniformmoments}
	\sup_n\E\left(\sum_j\supnorm{Z^{n,j}}{T}\right)^2<\infty,
	\qquad
	\sup_n\E\supnorm{Y^n}{T}^2<\infty,
\end{equation}
where the second assertion follows from \eqref{eq:growth}. Returning to
\eqref{eq:remainderL2}, we obtain
\begin{equation}\label{eq:remainderfinal}
	\left\lVert\sum_{k=0}^{{N_n}-1}\sum_j|r_k^j|\right\rVert_{{L^2(\Omega)}}
	\le C\sqrt h.
\end{equation}

Let $\eta_n(s)={t_k^n}$ for $s\in[{t_k^n},{t_{k+1}^n})$, and set
$\eta_n(T)=T$. Define the continuous adapted processes
\begin{equation}\label{eq:Jdefinition}
	J_t^{n,j}:=\int_0^t\sigma_j(Y_{\eta_n(s)}^n)\,dB_s^j
	+\int_0^t{\rho_j}(Y_{\eta_n(s)}^n)\,ds.
\end{equation}
The stochastic integrand has a predictable step-process version; its
values at the finitely many partition points are immaterial. At mesh
points, $J_{{t_i^n}}^{n,j}=M_i^j+A_i^j$. Therefore
\eqref{eq:martingalebounds} and \eqref{eq:remainderfinal} give
\begin{equation}\label{eq:griderror}
	\sum_j\E\max_{i\le {N_n}}
	|Z_{{t_i^n}}^{n,j}-J_{{t_i^n}}^{n,j}|^2\le Ch.
\end{equation}

To extend the estimates from mesh points to the entire interval, define
\[
W_n:=\max_{0\le k<N_n}\sup_{t_k^n\le s\le t_{k+1}^n}|B_s-B_{t_k^n}|.
\]
The reflection principle for each Brownian coordinate and a union bound
give, for $x>0$,
\[
\mathbb P(W_n>x)
\le \min\left\{1,\;4m{N_n} e^{-x^2/(2mh)}\right\}.
\]
Indeed, a vector with norm greater than $x$ has a coordinate of
absolute value greater than $x/\sqrt m$.
Integrating this tail bound, with
$x_0^2=2mh\log(4m{N_n})$, yields
\begin{equation}\label{eq:Brownianmaximum}
	\E W_n^2
	\le 2mh\bigl(\log(4m{N_n})+1\bigr)
	\le C h\log(2{N_n}).
\end{equation}
On each mesh interval, \eqref{eq:localZ} and boundedness of the
coefficients in \eqref{eq:Jdefinition} imply
\begin{align*}
	\sum_j\sup_{{t_k^n}\le s\le {t_{k+1}^n}}
	|Z_s^{n,j}-Z_{{t_k^n}}^{n,j}|&\le CW_n,\\
	\sum_j\sup_{{t_k^n}\le s\le {t_{k+1}^n}}
	|J_s^{n,j}-J_{{t_k^n}}^{n,j}|&\le CW_n+Ch.
\end{align*}
Combining these inequalities with \eqref{eq:griderror} and
\eqref{eq:Brownianmaximum}, we find
\begin{equation}\label{eq:epsilon}
	\varepsilon_n:=\sum_j\E\supnorm{Z^{n,j}-J^{n,j}}{T}^2
	\le C h_n\log(2N_n).
\end{equation}

We also need a fixed-time, rather than a maximal, oscillation estimate.
By \eqref{eq:localY} and \eqref{eq:localV},
\[
\sup_{{t_k^n}\le s\le {t_{k+1}^n}}|Y_s^n-y_k|
\le C\{R_k+h(1+S_k+R_k)\}.
\]
Taking second moments and using \eqref{eq:uniformmoments} yields
\begin{equation}\label{eq:fixedtimeoscillation}
	\sup_{0\le s\le T}\E|Y_s^n-Y_{\eta_n(s)}^n|^2\le Ch_n
\end{equation}
for all sufficiently fine meshes.

\subsection{Completion of the Proof}

Let $Y$ be the solution constructed in
Proposition~\ref{prop:wellposedness}, and let $Z^j$ be given by
\eqref{eq:Zlimit}. In particular,
\[
Z_t^j=\int_0^t\sigma_j(Y_s)\circ dB_s^j,
\qquad
Y={y_0}+\Dop Y+\sum_j{\mathcal K}Z^j.
\]
Define
\[
{\mathcal E_n}(t):=\E\supnorm{Y^n-Y}{t}^2.
\]
By \eqref{eq:stability} and \eqref{eq:epsilon},
\begin{equation}\label{eq:finalfirst}
	{\mathcal E_n}(t)\le C\varepsilon_n
	+C\sum_j\E\supnorm{J^{n,j}-Z^j}{t}^2.
\end{equation}
From \eqref{eq:Jdefinition} and \eqref{eq:Zlimit},
\begin{align*}
	J_t^{n,j}-Z_t^j
	={}&\int_0^t\bigl[\sigma_j(Y_{\eta_n(s)}^n)-\sigma_j(Y_s)\bigr]\,dB_s^j\\
	&+\int_0^t\bigl[{\rho_j}(Y_{\eta_n(s)}^n)-{\rho_j}(Y_s)\bigr]\,ds.
\end{align*}
All stochastic integrands in this comparison are adapted. Doob's
inequality, the It\^o isometry, and the Lipschitz properties of
$\sigma_j$ and ${\rho_j}$ give
\begin{align}\label{eq:finalsecond}
	\sum_j\E\supnorm{J^{n,j}-Z^j}{t}^2
	&\le C\int_0^t\E|Y_{\eta_n(s)}^n-Y_s|^2\,ds\notag\\
	&\le C\int_0^t
	\bigl[\E|Y_{\eta_n(s)}^n-Y_s^n|^2+{\mathcal E_n}(s)\bigr]\,ds\notag\\
	&\le Ch_n+C\int_0^t{\mathcal E_n}(s)\,ds,
\end{align}
where the last step uses \eqref{eq:fixedtimeoscillation}.
Combining \eqref{eq:finalfirst} and \eqref{eq:finalsecond},
\[
{\mathcal E_n}(t)\le C(\varepsilon_n+h_n)+C\int_0^t{\mathcal E_n}(s)\,ds.
\]
Gronwall's inequality and \eqref{eq:epsilon} now yield
\[
{\mathcal E_n}(T)\le C(\varepsilon_n+h_n)
\le C h_n\log(2N_n),
\]
which proves \eqref{eq:convergence}.

Since $h_n=T2^{-n}$,
\[
\sum_{n=0}^{\infty}h_n\log(2N_n)
=T\log 2\sum_{n=0}^{\infty}(n+1)2^{-n}<\infty.
\]
The finitely many coarse meshes also have finite second moments.
Consequently, for every $\delta>0$, Markov's inequality gives
\[
\sum_{n=0}^{\infty}
\mathbb P\bigl(\supnorm{Y^n-Y}{T}>\delta\bigr)<\infty.
\]
Apply the Borel--Cantelli lemma for $\delta=1/r$, $r\in\mathbb N$.
Intersecting the resulting events of probability one proves
\eqref{eq:asconvergence}. No independence of the errors
for different $n$ is needed. This completes the proof of
Theorem~\ref{thm:main}.

\begin{corollary}\label{cor:rate}
	For every $\alpha\in(0,1/2)$,
	\[
	h_n^{-\alpha}\supnorm{Y^n-Y}{T}\longrightarrow0
	\quad\text{almost surely}.
	\]
\end{corollary}
\begin{proof}
	For every $\delta>0$, \eqref{eq:convergence} gives
	\[
	\mathbb P\bigl(\supnorm{Y^n-Y}{T}>\delta h_n^\alpha\bigr)
	\le C\delta^{-2}h_n^{1-2\alpha}\log(2N_n).
	\]
	The right-hand side is summable in $n$. Apply the same
	Borel--Cantelli argument as above.
\end{proof}

\subsection{How the Stratonovich Enhancement Enters}\label{sec:enhancement}

This subsection interprets the preceding direct proof in terms of the
canonical Stratonovich enhancement of Brownian motion. Define
\begin{equation}\label{eq:enhancement}
	\mathbf B^{\mathrm S}=(B,\mathbb B^{\mathrm S}),
	\qquad
	\mathbb B_{s,t}^{\mathrm S,\ell j}
	:=\int_s^t(B_u^\ell-B_s^\ell)\circ dB_u^j,
	\qquad 0\le s\le t\le T.
\end{equation}
This is the geometric
Brownian enhancement; its symmetric part satisfies
\[
\mathbb B_{s,t}^{\mathrm S,\ell j}
+\mathbb B_{s,t}^{\mathrm S,j\ell}
=(B_t^\ell-B_s^\ell)(B_t^j-B_s^j).
\]
For related background on Brownian Volterra equations and their geometric
interpretation, see \cite{DeyaTindel,Kalinin}.

The canonical second level of the piecewise linear path is
\[
\mathbb B_{s,t}^{n,\ell j}
:=\int_s^t(B_u^{n,\ell}-B_s^{n,\ell})\,dB_u^{n,j}.
\]
On a single mesh interval it satisfies the exact identity
\begin{equation}\label{eq:PLleveltwo}
	\mathbb B_{{t_k^n},{t_{k+1}^n}}^{n,\ell j}
	=\frac12\Delta_kB^\ell\Delta_kB^j.
\end{equation}
Thus the quadratic term in \eqref{eq:rawexpansion} is precisely the
second-level contribution from the canonical lift of the approximating
path. Its conditional mean is ${\rho_j}(y_k)h$, and this term remains
in the limit. It produces the Stratonovich integral in
\eqref{eq:Zlimit} and \eqref{eq:conversion}.

The identification can also be checked directly with enhanced sums.
For a fixed terminal time $t$ and a deterministic partition $\pi$ of
$[0,t]$, consider
{
\begin{equation}\label{eq:enhancedsum}
\begin{aligned}
\mathcal S_\pi^j(t):=
\sum_{[u,v]\in\pi}\phi(t-u)\bigg[
&\sigma_j(Y_u)(B_v^j-B_u^j)\\
&+\sum_{\ell=1}^mD\sigma_j(Y_u)\phi(0)\sigma_\ell(Y_u)
\mathbb B_{u,v}^{\mathrm S,\ell j}\bigg].
\end{aligned}
\end{equation}
}
The coefficient ${\phi(0)\sigma_\ell(Y_u)}$ is the instantaneous
Brownian coefficient of $Y$, as seen in \eqref{eq:covariation}.
Write
\[
\mathbb B_{u,v}^{\mathrm S,\ell j}
=\int_u^v(B_r^\ell-B_u^\ell)\,dB_r^j
+\frac12\delta_{\ell j}(v-u).
\]
In \eqref{eq:enhancedsum}, the first-level sum converges in ${L^2(\Omega)}$ to
the It\^o convolution. The sum containing the centered It\^o iterated
integrals converges to zero in ${L^2(\Omega)}$: its summands are martingale
differences with bounded coefficients and second moments at most
$C(v-u)^2$, so the total second moment is bounded by
$Ct|\pi|$. The remaining ordinary Riemann sum converges in ${L^2(\Omega)}$ to
$\int_0^t{\phi(t-s)}{\rho_j}(Y_s)\,ds$. Consequently,
\[
\mathcal S_\pi^j(t)\longrightarrow
\int_0^t{\phi(t-s)}\sigma_j(Y_s)\circ dB_s^j
\quad\text{in }{L^2(\Omega)}\quad\text{as }|\pi|\to0.
\]
This verifies the Stratonovich interpretation of the enhanced sums
directly in the present Brownian setting.

The mixed coefficient in \eqref{eq:rawexpansion} is
$D\sigma_j(y_k)\phi(0)\sigma_\ell(y_k)$.
Only $j=\ell$ has a nonzero conditional mean because the Brownian
components are independent. The resulting correction in the equation
for $Y$ is
\[
\frac12\sum_{j=1}^m\int_0^t
\phi(t-s)D\sigma_j(Y_s)\phi(0)\sigma_j(Y_s)\,ds.
\]
Neither $\phi(0)$ nor $\phi(t-s)$ may in general be moved through
$D\sigma_j(Y_s)$. No symmetry, commutativity, or invertibility assumption
is imposed on the matrices. The correction cannot in general be
combined with the $\phi_0$-drift as a single convolution with $\phi_0$.
If $\phi(0)=0$, all the conversion terms vanish, even if $\phi$ is
nonzero away from the origin.
 
	The theorem proves strong convergence in
	$L^2(\Omega;C([0,T];\R^d))$ and uniform almost sure convergence
	for the specified dyadic piecewise linear approximations.
	It does not assert a deterministic continuity estimate for arbitrary
	enhanced driving paths. Lemma~\ref{lem:BV} supplies the
	continuously differentiable memory term needed for this Brownian
	proof, even when ${Dq}$ has atoms or a singular continuous part.
	The assumptions are sufficient conditions and are not claimed to be
	optimal.

\section{FNO and Universal Approximation}
\label{sec:FNO} 

In \cite{Li:21}, the authors introduce an neural network architecture called Fourier Neural Operator (FNO) designed for resolution free approximation of operators between spaces of periodic functions. Specifically, an FNO $\mathcal{N}$ has the form
\[
\mathcal{N}=\mathcal{Q}\circ\mathcal{L}_{L}\circ\mathcal{L}_{L-1}\cdots\circ\mathcal{L}_{1}\circ\mathcal{R},
\]
where $\mathcal{Q,R}$ are the usual projection and lift maps, respectively, and the $\mathcal{L}_i$ are what's called Fourier layers and have the form

\begin{equation}
    \label{Flayer}
    \mathcal{L}_{l}(v)(x)=\sigma\Big( W_{l}v(x)+b_{l}(x)+\mathcal{F}^{-1}\big(P_{l}\cdot\mathcal{F}v\big)(x)\Big).
\end{equation}
Here, the input $v$ is a function defined on the periodic domain $D$,
$\mathcal{F}$ denotes the Fourier transform on $D$, and $P_l$ is a learnable
matrix-valued multiplier acting on a finite number of Fourier modes.
Furthermore, $\sigma$ is a non-polynomial, globally Lipschitz $C^3$ activation function applied componentwise,
and $W_l$ and $b_l$ are a learnable weight matrix and bias function,
respectively.
The real-world implementation of the FNO is called the $\Psi$-FNO, where the Fourier transform in \eqref{Flayer} is computed through a discrete Fourier transform on the input's values on a lattice.In fact, it was shown in \cite{Kovachki:21} that the class of $\Psi$-FNOs can approximate any continuous operator from $H^s(\mathbb{T}^d,\mathbb{R}^d_1)\,(s> d/2)$ to $H^{s'}(\mathbb{T}^d,\mathbb{R}^{d_2})$ on a compact subset, where $\mathbb{T}^d$ is the $d$-dimensional torus. Furthermore, the $\Psi$-FNO is enriched to handle non-periodic inputs as well through zero-padding functions to enforce periodicity. However, zero-padding introduces discontinuities at boundaries, which hurts the smoothness condition for the approximation theorem to hold.

The solutions to SVIEs considered in this study can be thought of as the operator image of the driving Brownian motion. However, such operators are in general, measurable operators between the space of continuous non-periodic paths under 
practical assumptions. As such, approximation properties of $\Psi$-FNOs established by \cite{Kovachki:21} do not apply to the SVIEs at hand. See \cite{LKP:26} for a similar discussion for path-dependnt SDEs and fractional Brownian motion.

To resolve such obstructions, the notion of mirror-padded FNO (MFNO) and interpretation of inputs on grids as linear interpolated versions were introduced in \cite{LKP:26}. Before stating the definition of MFNO, we begin by denoting the pseudo-spectral projection of order $N$ by $\mathcal{I}_N$, which is used to describe DFT on grid size $2N+1$ . For a function $a:[0,T]\rightarrow\mathbb{R}^{d_a}$, define its mirror padding $\mathcal{M}a$ on $\mathbb{T}_{2T}$ by
\[
\mathcal{M}a(t)=
\begin{cases}
a(t),&0\leq t\leq T,\\
a(2T-t),&T<t\leq 2T.
\end{cases}
\]
We also define the truncation operator $\mathcal{T}$ by
\[
\mathcal{T}u=u|_{[0,T]}.
\]

\begin{definition}[Mirror-padded Fourier neural operator]
An MFNO of order $N$ is an operator of the form
\[
\mathcal{N}
=
\mathcal{T}\circ\mathcal{Q}\circ\mathcal{I}_N
\circ\mathcal{L}_L\circ\mathcal{I}_N\circ\cdots
\circ\mathcal{L}_1\circ\mathcal{I}_N
\circ\mathcal{R}\circ\mathcal{M},
\]
where $\mathcal{R}$, $\mathcal{Q}$ and $\mathcal{L}_1,\ldots,\mathcal{L}_L$
are the lift, projection and Fourier layers described above.
\end{definition}

Thus, an MFNO first extends a non-periodic input to a periodic function through mirror padding, applies the usual pseudo-spectral Fourier layers, and restricts the output back to $[0,T]$. The following universal approximation result was established in \cite{LKP:26}.

\begin{theorem}[Universal approximation by MFNOs]
\label{thm:uamfno}
Let
\[
\mathcal{G}:H^1([0,T],\mathbb{R}^{d_a})
\longrightarrow L^2([0,T],\mathbb{R}^{d_u})
\]
be a continuous operator and let
$K\subset H^1([0,T],\mathbb{R}^{d_a})$ be compact. Then, for every
$\epsilon>0$, there exist $N\in\mathbb{N}$ and an MFNO $\mathcal{N}$ of
order $N$ such that
\[
\sup_{a\in K}
\|\mathcal{G}(a)-\mathcal{N}(a)\|_{L^2([0,T],\mathbb{R}^{d_u})}
<\epsilon.
\]
\end{theorem}

\section{Approximation Theorem for SVIEs}
\label{sec:approx}

We consider the equation
\begin{equation}\label{eq:volterra-h1-driver}
Y_t
=
y_0
+
\int_0^t \phi_0(t-s)b(Y_s)\,ds
+
\int_0^t \phi(t-s)\sigma(Y_s)\,d\omega_s,
\qquad t\in[0,T],
\end{equation}
where $y_0\in\mathbb R^d$, $\omega\in H^1([0,T];\mathbb R^m)$,
\[
b:\mathbb R^d\to\mathbb R^d,
\qquad
\sigma:\mathbb R^d\to\mathbb R^{d\times m},
\qquad
\phi_0\to\mathbb R,\qquad\phi:[0,T]\to\mathbb{R}^{d\times d}
\]
In follows from assumptions \eqref{eq:assumptions} that $\phi_0\in L^1([0,T];\mathbb{R})$ and
$\phi\in L^2([0,T];\mathbb{R}^{d\times d})$, and that there exist positive constants
$L_b,L_\sigma,M_b,M_\sigma$ such that, for all $x,y\in\mathbb R^d$, 
\begin{equation}\label{eq:coeff-assumptions}
|b(x)-b(y)|\le L_b|x-y|,
\qquad
|b(0)|\le M_b,
\end{equation}
and
\begin{equation}\label{eq:sigma-assumptions}
|\sigma(x)-\sigma(y)|\le L_\sigma|x-y|,
\qquad
|\sigma(0)|\le M_\sigma.
\end{equation}
Here $|\cdot|$ denotes the Euclidean norm for vectors and the Frobenius norm for
matrices.

\begin{theorem}[Well-definedness of the Volterra solution map in $L^2$]
\label{thm:welldefVSM}
Under assumptions \eqref{eq:coeff-assumptions}--\eqref{eq:sigma-assumptions},
with $\phi_0\in L^1([0,T];\mathbb{R})$ and $\phi\in L^2([0,T];\mathbb {R}^{d\times d})$,
for every $y_0\in\mathbb{R}^d$ and $\omega\in H^1([0,T];\mathbb R^{m})$, equation
\eqref{eq:volterra-h1-driver} admits a unique solution
\[
Y^{y_0,\omega}\in L^2([0,T];\mathbb R^d).
\]
Consequently, the solution map
\[
F:H^1([0,T];\mathbb R^{d+m})\to L^2([0,T];\mathbb R^d),
\qquad
F(\gamma,\omega):=Y^{\gamma(0),\omega},
\]
is well-defined.
\end{theorem}

\begin{proof}
We fix $y_0\in\mathbb{R}^d$ and $\omega\in H^1([0,T];\mathbb R^m)$, then define $\Gamma(Y)$ by

\[
\Gamma(Y)(t)
:=
y_0
+
\int_0^t \phi_0(t-s)b(Y_s)\,ds
+
\int_0^t \phi(t-s)\sigma(Y_s)\,d\omega_s.
\]
for $Y\in L^2([0,T];\mathbb{R}^d)$.

We show that $\Gamma$ has a unique fixed point in $L^2([0,T];\mathbb R^d)$. To this end, we begin by defining a norm $\|\cdot\|_\lambda$ on $L^2([0,T];\mathbb{R}^d)$ by
\[
\|X\|_\lambda
:=
\left(
\int_0^T e^{-2\lambda t}|X_t|^2\,dt
\right)^{1/2}
\]
for each $\lambda>0$. Then, we have
\[
e^{-\lambda T}\|X\|_{L^2([0,T])}
\le
\|X\|_\lambda
\le
\|X\|_{L^2([0,T])}.
\]
That is, each $\|\cdot\|_\lambda$ is a norm equivalent to the standard norm $\|\cdot\|_{L^2([0,T];\mathbb{R}^d)}$. Also, denoting
\[
\phi_\lambda(r):=e^{-\lambda r}\phi(r),
\qquad
\phi_{0,\lambda}(r):=e^{-\lambda r}\phi_0(r),
\qquad r\in[0,T],
\]
we obtain
\begin{equation}
\label{Climit}
\|\phi_\lambda\|_{L^2([0,T])}\to0,
\qquad
\|\phi_{0,\lambda}\|_{L^1([0,T])}\to0
\qquad\text{as }\lambda\to\infty.
\end{equation}
by dominated convergence, since $\phi\in L^2([0,T])$ and $\phi_0\in L^1([0,T])$.

First, we note that $\Gamma$ maps $L^2([0,T];\mathbb R^d)$ into itself. Indeed, \eqref{eq:coeff-assumptions} implies
\[
|b(Y_s)|
\le
M_b+L_b|Y_s|,
\]
and therefore

\begin{equation}
\label{L21}
\left\|
\int_0^\cdot \phi_0(\cdot-s)b(Y_s)\,ds
\right\|_\lambda
\le
\|\phi_{0,\lambda}\|_{L^1([0,T])}
\left(M_bT^{1/2}+L_b\|Y\|_\lambda\right).
\end{equation}

Furthemore, we have
\[
|\sigma(Y_s)|
\le
M_\sigma+L_\sigma|Y_s|,
\]
and hence, applying Young's convolution inequality yields
\begin{equation}
\label{L22}
\begin{aligned}
\left\|
\int_0^\cdot \phi(\cdot-s)\sigma(Y_s)\,d\omega_s
\right\|_\lambda
&=\left\|e^{-\lambda\cdot}\int_0^\cdot \phi(\cdot-s)\sigma(Y_s)\dot\omega_s\,ds\right\|_{L^2([0,T])}\\
&=\left\|\int_0^\cdot e^{-\lambda(\cdot-s)}\phi(\cdot-s)\left(e^{-\lambda s}\sigma(Y_s)\dot\omega_s\right)\,ds\right\|_{L^2([0,T])}\\
&\le
\|\phi_\lambda\|_{L^2([0,T])}
\left\|
e^{-\lambda\cdot}|\sigma(Y)|\,|\dot \omega|
\right\|_{L^1(0,T)}
\\
&\le
\|\phi_\lambda\|_{L^2([0,T])}
\left(
M_\sigma T^{1/2}\|\dot \omega\|_{L^2([0,T])}
+
L_\sigma\|Y\|_\lambda\|\dot\omega\|_{L^2([0,T])}
\right).
\end{aligned}
\end{equation}
Since $Y\in L^2([0,T]\mathbb{R}^d)$ and $\omega\in H^1([0,T];\mathbb{R}^m)$ so that $\dot\omega\in L^2([0,T];\mathbb{R}^m)$, we conclude $\Gamma(Y)\in L^2([0,T];\mathbb R^d)$ from \eqref{L21} and \eqref{L22}.

Next, we show that $\Gamma$ is a contraction. Let $Y^1,Y^2\in L^2([0,T];\mathbb R^d)$, then by the Lipschitz assumptions, we have
\begin{equation}
\label{C0}
\begin{aligned}
|\Gamma (Y^1)(t)-\Gamma (Y^2)(t)|
&\le
L_b\int_0^t |\phi_0(t-s)|\,|Y^1_s-Y^2_s|\,ds
\\
&\quad+
L_\sigma\int_0^t
|\phi(t-s)|\,|\dot\omega_s|\,|Y^1_s-Y^2_s|\,ds.
\end{aligned}
\end{equation}
For the first term, applying Young's convolutional inequality  yields

\begin{equation}
\label{C1}
\begin{aligned}
\left\|
\int_0^\cdot |\phi_0(\cdot-s)|\,|Y^1_s-Y^2_s|\,ds
\right\|_\lambda
&\le \|\phi_{0,\lambda}\|_{L^1([0,T])}
\|e^{-\lambda\cdot}|Y^1-Y^2|\|_{L^2([0,T])}\\
&=\|\phi_{0,\lambda}\|_{L^1([0,T])}
\|Y^1-Y^2\|_\lambda.
\end{aligned}
\end{equation}

To deal with the second term, we denote
\[
J(t)
:=
\int_0^t |\phi(t-s)|\,|\dot\omega_s|\,|E_s|\,ds .
\]
where $E:=Y^1-Y^2$. It is straight forward that
\[
e^{-\lambda t}J(t)
=
\int_0^t
|\phi_\lambda(t-s)|\,|\dot\omega_s|\,e^{-\lambda s}|E_s|\,ds,
\]
and hence
\[
\|J\|_\lambda^2
=
\int_0^T
\left(
\int_0^t
|\phi_\lambda(t-s)|\,|\dot\omega_s|\,e^{-\lambda s}|E_s|\,ds
\right)^2dt.
\]
Applying the Cauchy--Schwartz inequality, we obtain
\[
\begin{aligned}
&\left(
\int_0^t
|\phi_\lambda(t-s)|\,|\dot\omega_s|\,e^{-\lambda s}|E_s|\,ds
\right)^2 \\
&\le
\left(
\int_0^t
|\phi_\lambda(t-s)|^2|\dot\omega_s|^2\,ds
\right)
\left(
\int_0^t
e^{-2\lambda s}|E_s|^2\,ds
\right).
\end{aligned}
\]
Since
\[
\int_0^t e^{-2\lambda s}|E_s|^2\,ds
\le
\int_0^T e^{-2\lambda s}|E_s|^2\,ds
=
\|E\|_\lambda^2,
\]
we arrive at
\[
\|J\|_\lambda^2
\le
\|E\|_\lambda^2
\int_0^T\int_0^t
|\phi_\lambda(t-s)|^2|\dot\omega_s|^2\,ds\,dt.
\]
By Fubini's theorem,
\[
\begin{aligned}
\int_0^T\int_0^t
|\phi_\lambda(t-s)|^2|\dot\omega_s|^2\,ds\,dt
&=
\int_0^T
|\dot\omega_s|^2
\int_s^T
|\phi_\lambda(t-s)|^2\,dt\,ds \\
&=
\int_0^T
|\dot\omega_s|^2
\int_0^{T-s}
|\phi_\lambda(r)|^2\,dr\,ds \\
&\le
\|\phi_\lambda\|_{L^2([0,T])}^2
\|\dot{\omega}\|_{L^2([0,T])}^2.
\end{aligned}
\]
Consequently,
\[
\|J\|_\lambda
\le
\|\phi_\lambda\|_{L^2([0,T])}
\|\dot\omega\|_{L^2([0,T])}
\|E\|_\lambda.
\]
That is,
\begin{equation}
\label{C2}
\left\|
\int_0^\cdot
|\phi(\cdot-s)|\,|\dot\omega_s|\,|Y^1_s-Y^2_s|\,ds
\right\|_\lambda
\le
\|\phi_\lambda\|_{L^2([0,T])}
\|\dot\omega\|_{L^2([0,T])}
\|Y^1-Y^2\|_\lambda.
\end{equation}
Thus, plugging in \eqref{C1} and \eqref{C2} to \eqref{C0}, we arrive at
\[
\|\Gamma(Y^1)-\Gamma(Y^2)\|_\lambda
\le
\left(
L_b\|\phi_{0,\lambda}\|_{L^1([0,T])}
+
L_\sigma
\|\phi_\lambda\|_{L^2([0,T])}
\|\dot{\omega}\|_{L^2([0,T])}
\right)
\|Y^1-Y^2\|_\lambda.
\]

By \eqref{Climit}, there exists a $\lambda_0>0$ such that
\[
L_b\|\phi_{0,\lambda}\|_{L^1([0,T])}
+
L_\sigma\|\dot{\omega}\|_{L^2([0,T])}
\|\phi_\lambda\|_{L^2([0,T])}
<1
\]
It is clear that for $\lambda>\lambda_0$, $\Gamma $ is a contraction on $L^2([0,T];\mathbb R^d)$ equipped with
the equivalent norm $\|\cdot\|_\lambda$. By Banach's fixed point theorem,
there exists a unique $Y^{y_0,\omega}\in L^2([0,T];\mathbb R^d)$ such that
\[
Y^{y_0,\omega}=\Gamma (Y^{y_0,\omega}).
\]
This is exactly equation \eqref{eq:volterra-h1-driver}, and therefore, the operator $F:H^1([0,T];\mathbb{R}^{d+m})\to L^2([0,T];\mathbb{R}^d)$ given by $F(\gamma,\omega):=Y^{\gamma(0),\omega}$
is well-defined.
\end{proof}

\begin{theorem}[Continuity of the Volterra solution map in $L^2$]
\label{thm:continuityVSM}
The operator
\[
F:H^1([0,T];\mathbb R^{d+m})
\to
L^2([0,T];\mathbb R^d)
\]
defined in Theorem \ref{thm:welldefVSM} is continuous.
\end{theorem}

\begin{proof}
Let
\[
(\gamma^n,\omega^n)\to(\gamma,\omega)
\qquad\text{in }H^1([0,T];\mathbb R^{d+m}),
\]
and denote
\[
Y^n:=F(\gamma^n,\omega^n),
\qquad
Y:=F(\gamma,\omega).
\]
By Theorem \ref{thm:welldefVSM}, $Y^n$ and $Y$ are well-defined elements of
$L^2([0,T];\mathbb R^d)$, and we prove our claim by showing that
\[
Y^n\to Y
\qquad\text{in }L^2([0,T];\mathbb R^d).
\]

We begin by noting that since
\[
\omega^n\to\omega
\qquad\text{in }H^1([0,T];\mathbb R^m),
\]
we have
\[
\dot\omega^n\to\dot\omega
\qquad\text{in }L^2([0,T];\mathbb R^m).
\]
Therefore, there exists a constant $M>0$ such that
\[
\|\dot\omega^n\|_{L^2([0,T])}\le M
\]
for every $n\in\mathbb N$. We use the norm $\|\cdot\|_\lambda$ and the functions $\phi_\lambda$ and $\phi_{0,\lambda}$ introduced
in the proof of Theorem \ref{thm:welldefVSM}. By \eqref{Climit}, we may choose
$\lambda_0>0$ sufficiently large such that
\begin{equation}
\label{eq:continuity-choice-lambda}
L_b\|\phi_{0,\lambda}\|_{L^1([0,T])}
+
L_\sigma M\|\phi_\lambda\|_{L^2([0,T])}
<
\frac{1}{2}.
\end{equation}
for any $\lambda>\lambda_0$

We denote
\[
E^n:=Y^n-Y.
\]
Subtracting the equations satisfied by $Y^n$ and $Y$, we obtain
\[
\begin{aligned}
E_t^n
&=
\gamma^n(0)-\gamma(0)
+
\int_0^t
\phi_0(t-s)\left(b(Y_s^n)-b(Y_s)\right)\,ds
\\
&\quad+
\int_0^t
\phi(t-s)
\left(
\sigma(Y_s^n)\dot\omega_s^n
-
\sigma(Y_s)\dot\omega_s
\right)\,ds.
\end{aligned}
\]
From
\[
\sigma(Y_s^n)\dot\omega_s^n
-
\sigma(Y_s)\dot\omega_s
=
\left(
\sigma(Y_s^n)-\sigma(Y_s)
\right)\dot\omega_s^n
+
\sigma(Y_s)
\left(
\dot\omega_s^n-\dot\omega_s
\right)
\]
and the Lipschitz assumptions, we have
\begin{equation}
\label{eq:continuity-pointwise}
\begin{aligned}
|E_t^n|
&\le
|\gamma^n(0)-\gamma(0)|
+
L_b\int_0^t|\phi_0(t-s)|\,|E_s^n|\,ds
\\
&\quad+
L_\sigma\int_0^t
|\phi(t-s)|\,|\dot\omega_s^n|\,|E_s^n|\,ds
\\
&\quad+
\int_0^t
|\phi(t-s)|\,|\sigma(Y_s)|
|\dot\omega_s^n-\dot\omega_s|\,ds.
\end{aligned}
\end{equation}

We begin by dealing with the first two terms. We have
\begin{equation}
\label{eq:continuity-first-term}
\begin{aligned}
\|\gamma^n(0)-\gamma(0)\|_\lambda
&=
\left(
\int_0^T e^{-2\lambda t}
|\gamma^n(0)-\gamma(0)|^2\,dt
\right)^{1/2}
\\
&\le
T^{1/2}|\gamma^n(0)-\gamma(0)|.
\end{aligned}
\end{equation}
Furthermore, since
\[
e^{-\lambda t}\int_0^t|\phi_0(t-s)|\,|E_s^n|\,ds
=
\int_0^t
|\phi_{0,\lambda}(t-s)|e^{-\lambda s}|E_s^n|\,ds,
\]
applying Young's convolution inequality yields

\begin{equation}
\label{eq:continuity-second-term}
\begin{aligned}
\left\|\int_0^\cdot|\phi_0(\cdot-s)|\,|E_s^n|\,ds\right\|_\lambda
&=
\left\|
\int_0^\cdot
|\phi_{0,\lambda}(\cdot-s)|e^{-\lambda s}|E_s^n|\,ds
\right\|_{L^2([0,T])}
\\
&\le
\|\phi_{0,\lambda}\|_{L^1([0,T])}
\|e^{-\lambda\cdot}|E^n|\|_{L^2([0,T])}
\\
&=
\|\phi_{0,\lambda}\|_{L^1([0,T])}\|E^n\|_\lambda.
\end{aligned}
\end{equation}

To deal with the third term, we begin by noting that
\[
e^{-\lambda t}
\int_0^t
|\phi(t-s)|\,|\dot\omega_s^n|\,|E_s^n|\,ds
=
\int_0^t
|\phi_\lambda(t-s)|e^{-\lambda s}
|\dot\omega_s^n|\,|E_s^n|\,ds.
\]
Hence, applying Young's convolution inequality, we obtain
\[
\left\|
\int_0^\cdot
|\phi(\cdot-s)|\,|\dot\omega_s^n|\,|E_s^n|\,ds
\right\|_\lambda
\le
\|\phi_\lambda\|_{L^2([0,T])}
\left\|
e^{-\lambda\cdot}|\dot\omega^n|\,|E^n|
\right\|_{L^1([0,T])}.
\]
Next, applying the Cauchy--Schwarz inequality yields
\[
\begin{aligned}
\left\|
e^{-\lambda\cdot}|\dot\omega^n|\,|E^n|
\right\|_{L^1([0,T])}
&=
\int_0^T
|\dot\omega_s^n|e^{-\lambda s}|E_s^n|\,ds
\\
&\le
\left(
\int_0^T|\dot\omega_s^n|^2\,ds
\right)^{1/2}
\left(
\int_0^T e^{-2\lambda s}|E_s^n|^2\,ds
\right)^{1/2}
\\
&=
\|\dot\omega^n\|_{L^2([0,T])}\|E^n\|_\lambda.
\end{aligned}
\]
Consequently,
\begin{equation}
\label{eq:continuity-third-term}
\left\|
\int_0^\cdot
|\phi(\cdot-s)|\,|\dot\omega_s^n|\,|E_s^n|\,ds
\right\|_\lambda
\le
M\|\phi_\lambda\|_{L^2([0,T])}\|E^n\|_\lambda.
\end{equation}

Now, we deal with the last term. Similar to before, we have
\[
\begin{aligned}
&e^{-\lambda t}
\int_0^t
|\phi(t-s)|\,|\sigma(Y_s)|
|\dot\omega_s^n-\dot\omega_s|\,ds
\\
&=
\int_0^t
|\phi_\lambda(t-s)|e^{-\lambda s}|\sigma(Y_s)|
|\dot\omega_s^n-\dot\omega_s|\,ds.
\end{aligned}
\]
Applying Young's convolution inequality yields
\[
\begin{aligned}
&\left\|
\int_0^\cdot
|\phi(\cdot-s)|\,|\sigma(Y_s)|
|\dot\omega_s^n-\dot\omega_s|\,ds
\right\|_\lambda
\\
&\le
\|\phi_\lambda\|_{L^2([0,T])}
\left\|
e^{-\lambda\cdot}|\sigma(Y)|
|\dot\omega^n-\dot\omega|
\right\|_{L^1([0,T])}.
\end{aligned}
\]
Next, applying the Cauchy--Schwarz inequality, we obtain
\[
\begin{aligned}
&\left\|
e^{-\lambda\cdot}|\sigma(Y)|
|\dot\omega^n-\dot\omega|
\right\|_{L^1([0,T])}
\\
&\le
\left\|e^{-\lambda\cdot}\sigma(Y)\right\|_{L^2([0,T])}
\|\dot\omega^n-\dot\omega\|_{L^2([0,T])}.
\end{aligned}
\]
Furthermore, since
\[
|\sigma(Y_s)|\le M_\sigma+L_\sigma|Y_s|,
\]
we have
\[
\begin{aligned}
\left\|e^{-\lambda\cdot}\sigma(Y)\right\|_{L^2([0,T])}
&\le
M_\sigma
\left(
\int_0^T e^{-2\lambda s}\,ds
\right)^{1/2}
+
L_\sigma\|Y\|_\lambda
\\
&\le
M_\sigma T^{1/2}+L_\sigma\|Y\|_\lambda.
\end{aligned}
\]
Consequently,
\begin{equation}
\label{eq:continuity-last-term}
\begin{aligned}
&\left\|
\int_0^\cdot
|\phi(\cdot-s)|\,|\sigma(Y_s)|
|\dot\omega_s^n-\dot\omega_s|\,ds
\right\|_\lambda
\\
&\le
\|\phi_\lambda\|_{L^2([0,T])}
\left(
M_\sigma T^{1/2}+L_\sigma\|Y\|_\lambda
\right)
\|\dot\omega^n-\dot\omega\|_{L^2([0,T])}.
\end{aligned}
\end{equation}

Combining the above estimates \eqref{eq:continuity-first-term},\eqref{eq:continuity-second-term}, \eqref{eq:continuity-third-term}, and \eqref{eq:continuity-last-term} with \eqref{eq:continuity-pointwise}, we arrive at
\[
\begin{aligned}
\|E^n\|_\lambda
&\le
T^{1/2}|\gamma^n(0)-\gamma(0)|
\\
&\quad+
\left(
L_b\|\phi_{0,\lambda}\|_{L^1([0,T])}
+
L_\sigma M\|\phi_\lambda\|_{L^2([0,T])}
\right)
\|E^n\|_\lambda
\\
&\quad+
\|\phi_\lambda\|_{L^2([0,T])}
\left(
M_\sigma T^{1/2}+L_\sigma\|Y\|_\lambda
\right)
\|\dot\omega^n-\dot\omega\|_{L^2([0,T])}.
\end{aligned}
\]
Thus, choosing $\lambda>\lambda_0$, we obtain
\begin{equation}
\label{eq:continuity-final}
\begin{aligned}
\|E^n\|_\lambda
&\le
2T^{1/2}|\gamma^n(0)-\gamma(0)|
\\
&\quad+
2\|\phi_\lambda\|_{L^2([0,T])}
\left(
M_\sigma T^{1/2}+L_\sigma\|Y\|_\lambda
\right)
\|\dot\omega^n-\dot\omega\|_{L^2([0,T])}.
\end{aligned}
\end{equation}
by \eqref{eq:continuity-choice-lambda}. Here, since $\gamma^n\to\gamma$ in $H^1([0,T];\mathbb R^d)$, it is straightforward that
\[
|\gamma^n(0)-\gamma(0)|\to0.
\]
Also, since $\omega^n\to\omega$ in $H^1([0,T];\mathbb R^m)$, we have
\[
\|\dot\omega^n-\dot\omega\|_{L^2([0,T])}\to0.
\]
Thus, we conclude
\[
\|E^n\|_\lambda\to0
\]
from \eqref{eq:continuity-final}, and therefore, the operator $F$ is continuous.
\end{proof}

\begin{theorem}[MFNO approximation for SVIEs]
\label{thm:MFNOapprox}
Suppose that $b$, $\sigma$, $\phi_0$ and $\phi$ satisfy
the assumptions in \eqref{eq:assumptions}.
Let $y_0\in \mathbb{R}^d$ and let $Y^{y_0}$
be the solution to \eqref{eq:stratonovich}.
Moreover, let $B^n$ be the non-adapted piecewise linear
interpolation of $B$ with respect to the dyadic partition
\[
\pi_n
:=
\{t_k^n:0\le k\le 2^n\},
\qquad
t_k^n:=kT2^{-n},
\qquad n\in\mathbb N.
\]
Then, for any $\epsilon,\epsilon'>0$ and $M\in\mathbb N$,
there exist $N,M_0\in\mathbb N$, an event $D$ and an MFNO
$\mathcal N:H^1([0,T],\mathbb R^{d+m})
\to L^2([0,T],\mathbb R^d)$ of order $N$ such that
$\mathbb P(D)>1-\epsilon'$ and
\[
\left(
\mathbb E\left[
\|Y^{y_0}-\mathcal N(y_0,B^n)\|_{L^2}^2
\mid D
\right]
\right)^{1/2}
<\epsilon
\]
whenever $M_0\le n\le M_0+M$.
Here $y_0$ is identified with the constant path $t\mapsto y_0$.
\end{theorem}

\begin{proof}
We prove this theorem in several steps.
For $\gamma\in H^1([0,T],\mathbb R^d)$ and
$\zeta\in H^1([0,T],\mathbb R^m)$, we consider the Volterra equation
\begin{equation}
\label{WZap1}
\tilde Y(t)
=
\gamma(0)
+
\int_0^t\phi_0(t-s)b(\tilde Y(s))\,ds
+
\int_0^t\phi(t-s)\sigma(\tilde Y(s))\,d\zeta(s),
\qquad t\in[0,T].
\end{equation}
This equation has a unique solution $\tilde Y^{\gamma,\zeta}$
by Theorem~\ref{thm:welldefVSM}.
Also, if we define
$F:H^1([0,T],\mathbb R^{d+m})
\to L^2([0,T],\mathbb R^d)$ by
$F(\gamma,\zeta)=\tilde Y^{\gamma,\zeta}$,
then the map $F$ is continuous by
Theorem~\ref{thm:continuityVSM}. Furthermore, by uniqueness,
\[
F(y_0,B^n)=Y^{n,y_0}.
\]

It suffices to consider $0<\epsilon'<1$.
By Theorem~\ref{thm:main}, there exists a constant $C>0$
such that
\[
\mathbb E\|Y^{y_0}-F(y_0,B^n)\|_{L^2}^2
\le
T\mathbb E\|Y^{y_0}-Y^{n,y_0}\|_{L^\infty}^2
\le
CT|\pi_n|\log\!\left(\frac{2T}{|\pi_n|}\right)
\]
for all sufficiently large $n$.
Also, since $|\pi_n|\to0$ as $n\to\infty$, there exists
$M_0\in\mathbb N$ such that
\[
\left(
\mathbb E\|Y^{y_0}-F(y_0,B^n)\|_{L^2}^2
\right)^{1/2}
<
\frac{\epsilon(1-\epsilon')^{1/2}}{2}
\]
for all $n\ge M_0$.

We denote by $\mathcal P\ell(\pi_n,\mathbb R^m)$
the finite-dimensional subspace of
$H^1([0,T],\mathbb R^m)$ consisting of functions that
are piecewise linear with respect to the partition $\pi_n$.
Choose $R>\sqrt T\,|y_0|$, so that $\|y_0\|_{H^1}<R$.
Next, let
\[
\epsilon'':=\frac{\epsilon'}{2(M+1)},
\]
and let $R_{M_0+M,\epsilon''}$ be the constant given in
\citet[Lemma~1]{LKP:26}. Then, we have
\[
\mathbb P\left(
\|B^n\|_{H^1}\le R_{M_0+M,\epsilon''}
\right)\ge1-\epsilon''
\]
for all $n\le M_0+M$.
Also, let
\[
\begin{aligned}
K:=\Big\{
(x,\zeta)\in H^1([0,T],\mathbb R^{d+m})
\,\Big|\,
&x\text{ is constant},\quad \|x\|_{H^1}\le R,\\
&\|\zeta\|_{H^1}\le R_{M_0+M,\epsilon''},\\
&\zeta\in\mathcal P\ell(\pi_{M_0+M},\mathbb R^m)
\Big\},
\end{aligned}
\]
and let
\[
D:=
\left\{
\omega\in\Omega:
(y_0,B^n(\omega))\in K
\text{ for }M_0\le n\le M_0+M
\right\}.
\]
Then $K$ is a compact subset of
$H^1([0,T],\mathbb R^{d+m})$, since it is a closed bounded
subset of a finite-dimensional subspace.
Moreover, the dyadic partitions are nested, so that
\[
\mathcal P\ell(\pi_n,\mathbb R^m)
\subset
\mathcal P\ell(\pi_{M_0+M},\mathbb R^m),
\qquad n\le M_0+M.
\]
Consequently, by the preceding probability bound and the union bound,
\[
\mathbb P(D)
\ge1-(M+1)\epsilon''
=1-\frac{\epsilon'}{2}
>1-\epsilon',
\]
which implies
\[
\begin{aligned}
&\left(
\mathbb E\left[
\|Y^{y_0}-F(y_0,B^n)\|_{L^2}^2\mid D
\right]
\right)^{1/2}
\\
&\qquad\le
\frac{
\left(
\mathbb E\|Y^{y_0}-F(y_0,B^n)\|_{L^2}^2
\right)^{1/2}
}{\mathbb P(D)^{1/2}}
<
\frac{\epsilon}{2}
\end{aligned}
\]
whenever $M_0\le n\le M_0+M$.

By the universal approximation theorem for MFNOs
 (Theorem \ref{thm:uamfno}), there exists an MFNO
$\mathcal N:H^1([0,T],\mathbb R^{d+m})
\to L^2([0,T],\mathbb R^d)$ of order $N$ such that
\[
\sup_{(x,\zeta)\in K}
\|F(x,\zeta)-\mathcal N(x,\zeta)\|_{L^2}
<
\frac{\epsilon}{2},
\]
so that
\[
\|F(y_0,B^n)-\mathcal N(y_0,B^n)\|_{L^2}
<
\frac{\epsilon}{2}
\]
on the set $D$.
Thus, for $M_0\le n\le M_0+M$,
\begin{align*}
&\left(
\mathbb E\left[
\|Y^{y_0}-\mathcal N(y_0,B^n)\|_{L^2}^2\mid D
\right]
\right)^{1/2}
\\
&\le
\left(
\mathbb E\left[
\left(
\|Y^{y_0}-F(y_0,B^n)\|_{L^2}
+
\|F(y_0,B^n)-\mathcal N(y_0,B^n)\|_{L^2}
\right)^2
\mid D
\right]
\right)^{1/2}
\\
&\le
\left(
\mathbb E\left[
\|Y^{y_0}-F(y_0,B^n)\|_{L^2}^2\mid D
\right]
\right)^{1/2}
+
\left(
\mathbb E\left[
\|F(y_0,B^n)-\mathcal N(y_0,B^n)\|_{L^2}^2\mid D
\right]
\right)^{1/2}
\\
&<
\frac{\epsilon}{2}+\frac{\epsilon}{2}
=
\epsilon.
\end{align*}
This completes the proof.
\end{proof}

Note that our theorem guarantees approximation for a fixed range of resolutions with an MFNO with fixed weights on an event set of sufficiently large probability, but not necessarily for arbitrarily high resolution.
\section{Experiment}
\label{sec:experi}

\subsection{Problem Setup}
Each task is a one-dimensional stochastic Volterra equation on $[0,T]$ with
$T=12.8$. The initial value is sampled as
\[
    Y_0 \sim \mathrm{Unif}(0.5,1.5),
\]
and the noise coefficient is $\sigma=0.35$. The learning input is the pair
\[
    X(t) = (B_t, Y_0),
\]
where $B_t$ is the Brownian path and $Y_0$ is repeated as a constant input
channel over the time grid. The target is the path $Y(t)$.

The continuous model being approximated is
\[
    Y_t
    =
    Y_0
    +
    \int_0^t \phi_0(t-s)b(Y_s)\,ds
    +
    \sigma \int_0^t K(t-s)\,dB_s .
\]
For these experiments, $\phi_0=\phi=K$.
The three tasks are listed in Table~\ref{tab:tasks}. 
Tasks 2 and 3 use regularized rough kernels. The convergence result applies for each fixed $\epsilon>0$, but does not establish convergence for the singular kernels obtained by setting $\epsilon=0$ or provide error bounds uniform in $\epsilon$ as $\epsilon\searrow0$.
\begin{table}[htbp]
\centering
\caption{Volterra tasks.}
\label{tab:tasks}
\begin{tabular}{lllllll}
\toprule
Task & Kernel $K(t)$ & Drift $b$ & $H$ & $\epsilon$ & $b$ parameter \\
\midrule
\TaskOne &  $\exp(-2t)$ & $b(y)=y$ & -- & -- & $1.0$ \\
\TaskTwo & $(t+\epsilon)^{H-1/2}$ & $b(y)=by$ & $0.40$ & $\num{1e-4}$ & $0.08$ \\
\TaskThree & $(t+\epsilon)^{H-1/2}$ & $b(y)=b$ & $0.40$ & $\num{1e-4}$ & $0.08$ \\
\bottomrule
\end{tabular}
\end{table}

\subsection{Data Preparation}
The training grid resolution is $128$, so each training path has $129$ time
points. The training set contains $8192$ paths per task, and each test set
contains $2048$ paths. Test data were generated at resolutions
\[
    128,160,192,256,320,384,512,640,832,1024.
\]

To reduce time-discretization error, each sample is first simulated on an
internal time grid that is considerably finer than the grid stored in the
dataset. Here, a resolution of $r$ means $r$ time intervals, or $r+1$ time
points. The refinement factor is chosen as
\[
    \mathrm{skip}(r) =
    r\,\left\lceil \frac{128\cdot 32}{r} \right\rceil,\qquad
    N_{\mathrm{fine}}(r)=r\,\mathrm{skip}(r),
\]
where $N_{\mathrm{fine}}(r)$ is the number of intervals used for the simulation.
This choice guarantees $N_{\mathrm{fine}}(r)\geq 4096$. After the simulation,
every $\mathrm{skip}(r)$-th point is retained, including both endpoints. The
result is the requested uniform grid of $r+1$ points. Thus, changing $r$
changes the resolution of the saved sample, while every sample is still
computed from a simulation using at least $4096$ internal time intervals.

The discrete noise term is generated from Brownian increments
$\Delta B_n\sim N(0,\Delta t)$ and convolved with $K$ using FFT-based one
dimensional linear convolution. For linear drift, the discrete Volterra equation
is solved using the resolvent recursion
\[
    r_0=1,\qquad
    r_n=\sum_{m=1}^{n} b\,\Delta t\,K_m\,r_{n-m}.
\]
For the constant-drift task, the deterministic drift contribution is computed as
$b\,\Delta t\sum_{m=1}^{n}K_m$.

\subsection{Models and Hyperparameter Search}
The compared models are the mirror-padded Fourier neural operator
(MFNO) \citep{LKP:26}, the zero-padded Fourier neural operator (ZFNO),
the Fourier neural operator (FNO) \citep{Li:21}, long short-term memory
(LSTM) networks \citep{HochreiterSchmidhuber1997}, temporal convolutional
networks (TCNs) \citep{BaiKolterKoltun2018}, and the deep operator network
(DeepONet) \citep{Lu2021DeepONet}. Here, ZFNO denotes the FNO variant
with zero padding. Hyperparameter search was performed separately for every task-model pair. The
search criterion was validation relative $\RelLtwo$ error as implemented in the code.
The training set was split into a training portion and a validation tail of
$1024$ paths.

The search was two-stage.
\begin{enumerate}
    \item Stage 1 searched the Cartesian product of $5$ architecture candidates
    and learning rates $\{10^{-3},5\cdot10^{-4},2\cdot10^{-4}\}$. Each candidate
    was trained for $50$ epochs with the scheduler disabled. The best
    architecture and learning rate were selected by validation $\RelLtwo$.
    \item Stage 2 fixed the selected architecture and learning rate, then tuned
    StepLR schedulers over
    \[
        (25,0.9),\ (50,0.9),\ (50,0.7),\ (100,0.9),\ (100,0.7),
    \]
    where each pair denotes $(\mathrm{step\ size},\gamma)$. Scheduler candidates
    were trained for $200$ epochs, again selected by validation $\RelLtwo$.
\end{enumerate}
The final training run uses the chosen architecture, learning rate, and
scheduler.

Tables~\ref{tab:selected-hyperparams-task1}--\ref{tab:selected-hyperparams-task3}
list the selected hyperparameters used by the training code for this run.
For FNO-family entries, $m,w,d,k$ denote Fourier modes, width, depth, and kernel
size. For TCN, $w$ is hidden width. For LSTM, $h$ is hidden dimension. For
DeepONet, $p$ is the latent product width. Learning rates and validation errors
are written in standard scientific notation, with no trailing zeros in the
mantissa.

{\scriptsize
\begin{table}[htbp]
\centering
\caption{Selected hyperparameters for \TaskOne.}
\label{tab:selected-hyperparams-task1}
\begin{tabular}{lp{6.0cm}ccc}
\toprule
Model & Architecture & LR & StepLR $(s,\gamma)$ & Val. MSE \\
\midrule
MFNO & \texttt{m=64, w=64, d=6, k=1} & \sci{5}{-4} & $(25,0.9)$ & \sci{2.44151}{-4} \\
ZFNO & \texttt{m=64, w=64, d=6, k=1} & \sci{1}{-3} & $(100,0.7)$ & \sci{1.28829}{-4} \\
FNO & \texttt{m=64, w=64, d=6, k=1} & \sci{5}{-4} & $(100,0.7)$ & \sci{3.69936}{-4} \\
LSTM & \texttt{h=128, layers=1, drop=0.0} & \sci{1}{-3} & $(25,0.9)$ & \sci{9.78445}{-4} \\
TCN & \texttt{w=512, blocks=6, k=2, drop=0.2} & \sci{1}{-3} & $(25,0.9)$ & \sci{4.50948}{-4} \\
DeepONet & \texttt{p=300, branch=3, trunk=2, w=128} & \sci{2}{-4} & $(25,0.9)$ & \sci{1.52127}{-2} \\
\bottomrule
\end{tabular}
\end{table}
}

{\scriptsize
\begin{table}[htbp]
\centering
\caption{Selected hyperparameters for \TaskTwo.}
\label{tab:selected-hyperparams-task2}
\begin{tabular}{lp{6.0cm}ccc}
\toprule
Model & Architecture & LR & StepLR $(s,\gamma)$ & Val. MSE \\
\midrule
MFNO & \texttt{m=64, w=64, d=6, k=1} & \sci{2}{-4} & $(50,0.7)$ & \sci{7.59577}{-4} \\
ZFNO & \texttt{m=64, w=64, d=6, k=1} & \sci{2}{-4} & $(50,0.7)$ & \sci{4.67849}{-4} \\
FNO & \texttt{m=32, w=64, d=6, k=1} & \sci{2}{-4} & $(25,0.9)$ & \sci{8.71486}{-4} \\
LSTM & \texttt{h=128, layers=1, drop=0.0} & \sci{1}{-3} & $(50,0.9)$ & \sci{2.07717}{-3} \\
TCN & \texttt{w=512, blocks=6, k=2, drop=0.2} & \sci{1}{-3} & $(100,0.7)$ & \sci{1.86754}{-3} \\
DeepONet & \texttt{p=384, branch=4, trunk=3, w=256} & \sci{2}{-4} & $(50,0.7)$ & \sci{6.95737}{-2} \\
\bottomrule
\end{tabular}
\end{table}
}

{\scriptsize
\begin{table}[htbp]
\centering
\caption{Selected hyperparameters for \TaskThree.}
\label{tab:selected-hyperparams-task3}
\begin{tabular}{lp{6.0cm}ccc}
\toprule
Model & Architecture & LR & StepLR $(s,\gamma)$ & Val. MSE \\
\midrule
MFNO & \texttt{m=64, w=64, d=6, k=1} & \sci{2}{-4} & $(50,0.7)$ & \sci{5.42639}{-4} \\
ZFNO & \texttt{m=64, w=64, d=6, k=1} & \sci{2}{-4} & $(100,0.9)$ & \sci{4.81996}{-4} \\
FNO & \texttt{m=32, w=32, d=6, k=1} & \sci{2}{-4} & $(25,0.9)$ & \sci{7.18237}{-4} \\
LSTM & \texttt{h=128, layers=1, drop=0.0} & \sci{1}{-3} & $(50,0.9)$ & \sci{2.60186}{-3} \\
TCN & \texttt{w=128, blocks=6, k=3, drop=0.2} & \sci{1}{-3} & $(25,0.9)$ & \sci{1.47409}{-3} \\
DeepONet & \texttt{p=300, branch=3, trunk=2, w=128} & \sci{2}{-4} & $(50,0.7)$ & \sci{6.71098}{-2} \\
\bottomrule
\end{tabular}
\end{table}
}

\subsection{Training Procedure}
Final training used seeds $1,\ldots,10$ for every task-model pair. The optimizer is Adam. The loss function is mean squared error,
\[
    \mathcal{L}_{\mathrm{MSE}}
    =
    \frac{1}{M_{\text{train}}}\sum^{M_{\text{train}}}_{i=1} \lVert \widehat{Y}^{(i)}-Y^{(i)}\rVert_2^2,
\]
computed over the tensor entries in each mini-batch, where $M_{\text{train}}$ is the number of sample paths. Final training uses $200$
epochs and batch size $256$. Weight decay is $\num{3e-3}$ for MFNO, ZFNO, FNO,
LSTM, and TCN, and $\num{3e-4}$ for DeepONet. The StepLR scheduler is the
task/model-specific scheduler selected in
Tables~\ref{tab:selected-hyperparams-task1}--\ref{tab:selected-hyperparams-task3}.

\subsection{Evaluation Protocol}
Evaluation uses the saved seed checkpoints. At resolution $128$, every model is
evaluated on the corresponding test set for each task. For higher resolutions,
only MFNO, ZFNO, and FNO are evaluated because the other architectures are tied
to the training time grid in this implementation.

The code computes one relative $\RelLtwo$-like error and one relative
$\RelLinf$ error per
seed:
\[
    \RelLtwo(Y,\widehat{Y})
    =
    \frac{1}{M_{\text{test}}}\sum^{M_{\text{test}}}_{i=1}\sqrt{\frac{\sum_j |Y^{(i)}_j-\widehat{Y}^{(i)}_j|^2}
         {\sum_j |Y^{(i)}_j|^2}},
    \qquad
    \RelLinf(Y,\widehat{Y})
    =
     \frac{1}{M_{\text{test}}}\sum^{M_{\text{test}}}_{i=1}\frac{\max_j |Y^{(i)}_j-\widehat{Y}^{(i)}_j|}
         {\max_j |Y^{(i)}_j|},
\]
where $M_{\text{test}}$ is the number of test samples. For every task-model-resolution combination, the table
reports the mean and sample standard deviation over seeds. If a seed produces
NaN or non-finite values, the mean and standard deviation are computed over the
finite seed values only, and the number of NaNs is reported in the relevant
table caption.

\section{Results}
\label {sec:results}

Training loss over iterations is shown in the following figure (Fig~\ref{fig:training-loss}).
\begin{figure}
\centering
\includegraphics[width=0.85\textwidth]{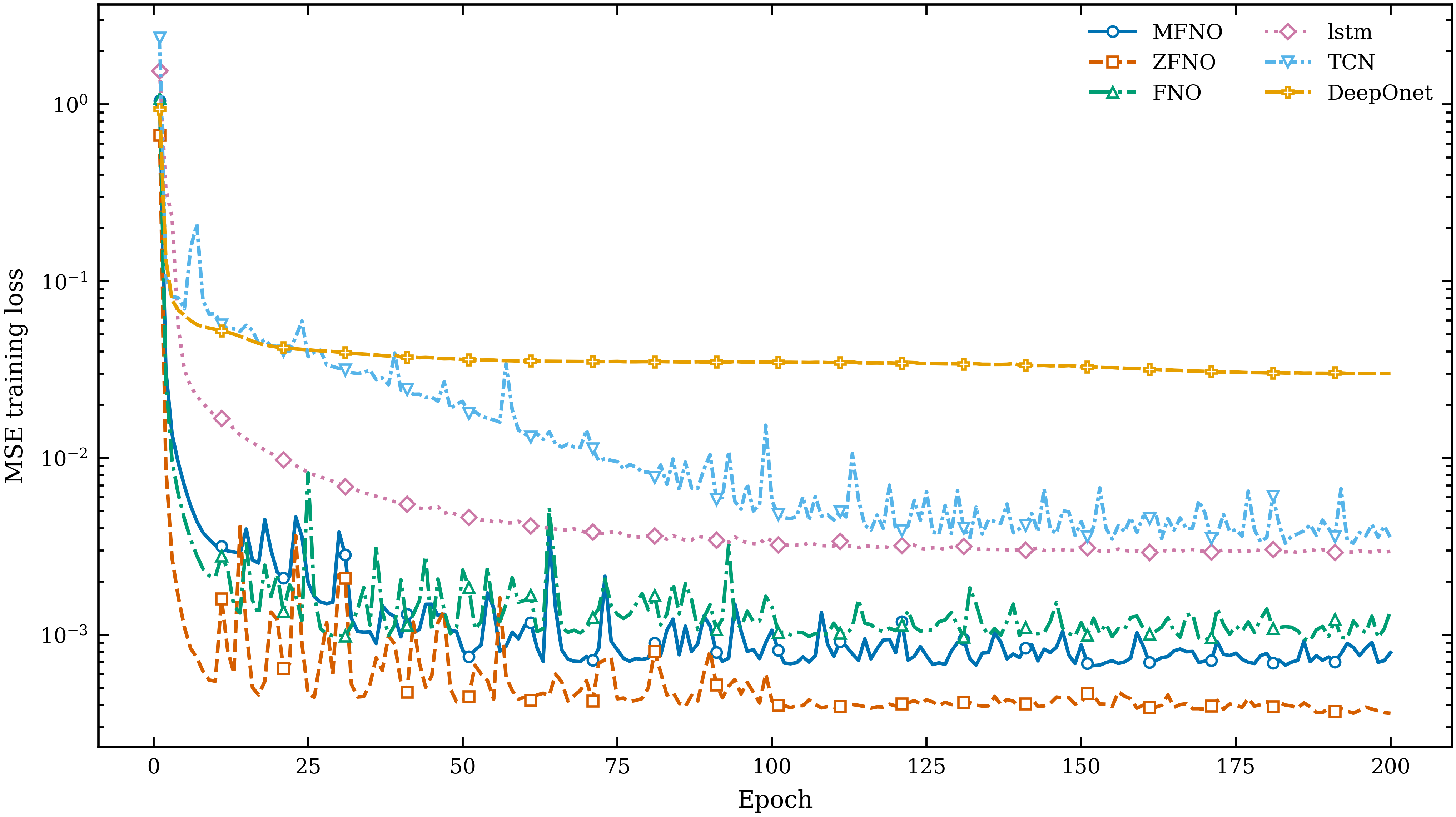}
\caption{Training loss comparison for \TaskOne, seed 1.}
\label{fig:training-loss}
\end{figure}
\subsection{Results at Training Resolution}
Tables~\ref{tab:base-errors-l2} and~\ref{tab:base-errors-linf} give the
paper-facing comparison at resolution $128$. Tasks form the column axis, and
boldface marks the smallest mean error in each task column. Means and sample
standard deviations are over seeds, excluding non-finite values where noted in
the captions.

{\scriptsize
\begin{table}[htbp]
\centering
\caption{Resolution-$128$ relative $\RelLtwo$ errors. TCN produced no NaNs on
\TaskOne, 5 NaNs out of 10 seeds on \TaskTwo, and 8 NaNs out of 10 seeds on \TaskThree; the TCN
statistics for \TaskTwo\ and \TaskThree\ use only the finite seeds.}
\label{tab:base-errors-l2}
\footnotesize
\begin{tabular}{lccc}
\toprule
Model & \TaskOne & \TaskTwo & \TaskThree \\
\midrule
MFNO & \sci{1.451}{-2} $\pm$ \sci{4.150}{-4} & \sci{2.020}{-2} $\pm$ \sci{5.575}{-4} & \sci{1.978}{-2} $\pm$ \sci{4.828}{-4} \\
ZFNO & \bsci{1.050}{-2} $\pm$ \sci{5.232}{-4} & \bsci{1.711}{-2} $\pm$ \sci{2.779}{-4} & \bsci{1.855}{-2} $\pm$ \sci{4.354}{-4} \\
FNO & \sci{1.720}{-2} $\pm$ \sci{1.153}{-3} & \sci{2.363}{-2} $\pm$ \sci{1.329}{-3} & \sci{2.330}{-2} $\pm$ \sci{6.305}{-4} \\
LSTM & \sci{2.827}{-2} $\pm$ \sci{1.932}{-3} & \sci{3.797}{-2} $\pm$ \sci{3.254}{-3} & \sci{4.505}{-2} $\pm$ \sci{2.773}{-3} \\
TCN & \sci{2.111}{-2} $\pm$ \sci{2.770}{-3} & \sci{4.873}{-2} $\pm$ \sci{1.529}{-2} & \sci{3.167}{-2} $\pm$ \sci{4.304}{-3} \\
DeepONet & \sci{1.026}{-1} $\pm$ \sci{3.474}{-3} & \sci{1.664}{-1} $\pm$ \sci{3.852}{-3} & \sci{1.818}{-1} $\pm$ \sci{6.010}{-3} \\
\bottomrule
\end{tabular}
\end{table}
}

{\scriptsize
\begin{table}[htbp]
\centering
\caption{Resolution-$128$ relative $\RelLinf$ errors. TCN produced no NaNs on
\TaskOne, 5 NaNs out of 10 seeds on \TaskTwo, and 8 NaNs out of 10 seeds on \TaskThree;
the TCN statistics for \TaskTwo\ and \TaskThree\ use only the finite seeds.}
\label{tab:base-errors-linf}
\footnotesize
\begin{tabular}{lccc}
\toprule
Model & \TaskOne & \TaskTwo & \TaskThree \\
\midrule
MFNO & \sci{5.009}{-2} $\pm$ \sci{1.756}{-3} & \sci{3.2}{-2} $\pm$ \sci{4.899}{-4} & \sci{3.472}{-2} $\pm$ \sci{7.024}{-4} \\
ZFNO & \bsci{3.234}{-2} $\pm$ \sci{1.33}{-3} & \bsci{2.676}{-2} $\pm$ \sci{4.723}{-4} & \bsci{3.107}{-2} $\pm$ \sci{5.798}{-4} \\
FNO & \sci{4.942}{-2} $\pm$ \sci{1.473}{-3} & \sci{3.943}{-2} $\pm$ \sci{1.531}{-3} & \sci{5.144}{-2} $\pm$ \sci{1.864}{-3} \\
LSTM & \sci{1.024}{-1} $\pm$ \sci{4.219}{-3} & \sci{9.632}{-2} $\pm$ \sci{3.493}{-3} & \sci{1.021}{-1} $\pm$ \sci{4.259}{-3} \\
TCN & \sci{4.756}{-2} $\pm$ \sci{4.487}{-3} & \sci{8.505}{-2} $\pm$ \sci{1.65}{-2} & \sci{4.91}{-2} $\pm$ \sci{2.328}{-3} \\
DeepONet & \sci{2.045}{-1} $\pm$ \sci{5.952}{-3} & \sci{2.433}{-1} $\pm$ \sci{5.031}{-3} & \sci{2.783}{-1} $\pm$ \sci{7.847}{-3} \\
\bottomrule
\end{tabular}
\end{table}
}

At the training resolution, ZFNO has the smallest reported $\RelLtwo$ and
$\RelLinf$ errors on all three tasks. MFNO is consistently better than the
plain FNO in $\RelLtwo$ and is also better than FNO in $\RelLinf$ on the two
rough tasks. DeepONet has the largest errors in these runs. The TCN performs
competitively on \TaskOne\ but exhibits non-finite failures on \TaskTwo\ and
\TaskThree.

\subsection{Resolution Generalization}
The higher-resolution experiment evaluates MFNO, ZFNO, and FNO at resolutions
$160$ through $1024$, excluding the training resolution $128$. The full numerical
table is given in Appendix~\ref{app:resolution-table}. Figure~\ref{fig:resolution-plots}
shows the same results as separate relative $\RelLtwo$ and relative $\RelLinf$ plots for
each task.

\begin{figure}[htbp]
\centering
\begin{subfigure}{0.48\textwidth}
    \includegraphics[width=\linewidth]{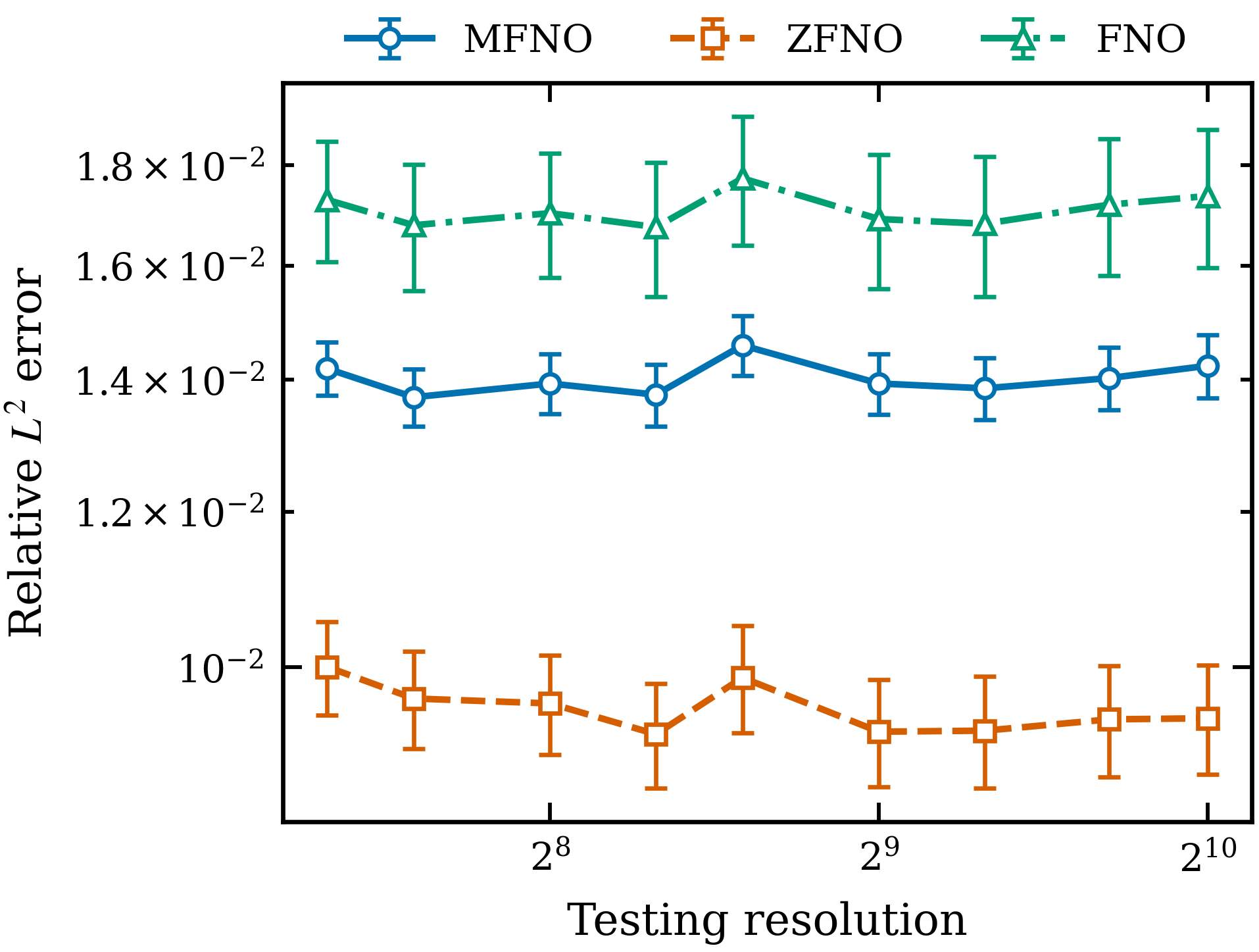}
    \caption{\TaskOne, $\RelLtwo$}
\end{subfigure}
\begin{subfigure}{0.48\textwidth}
    \includegraphics[width=\linewidth]{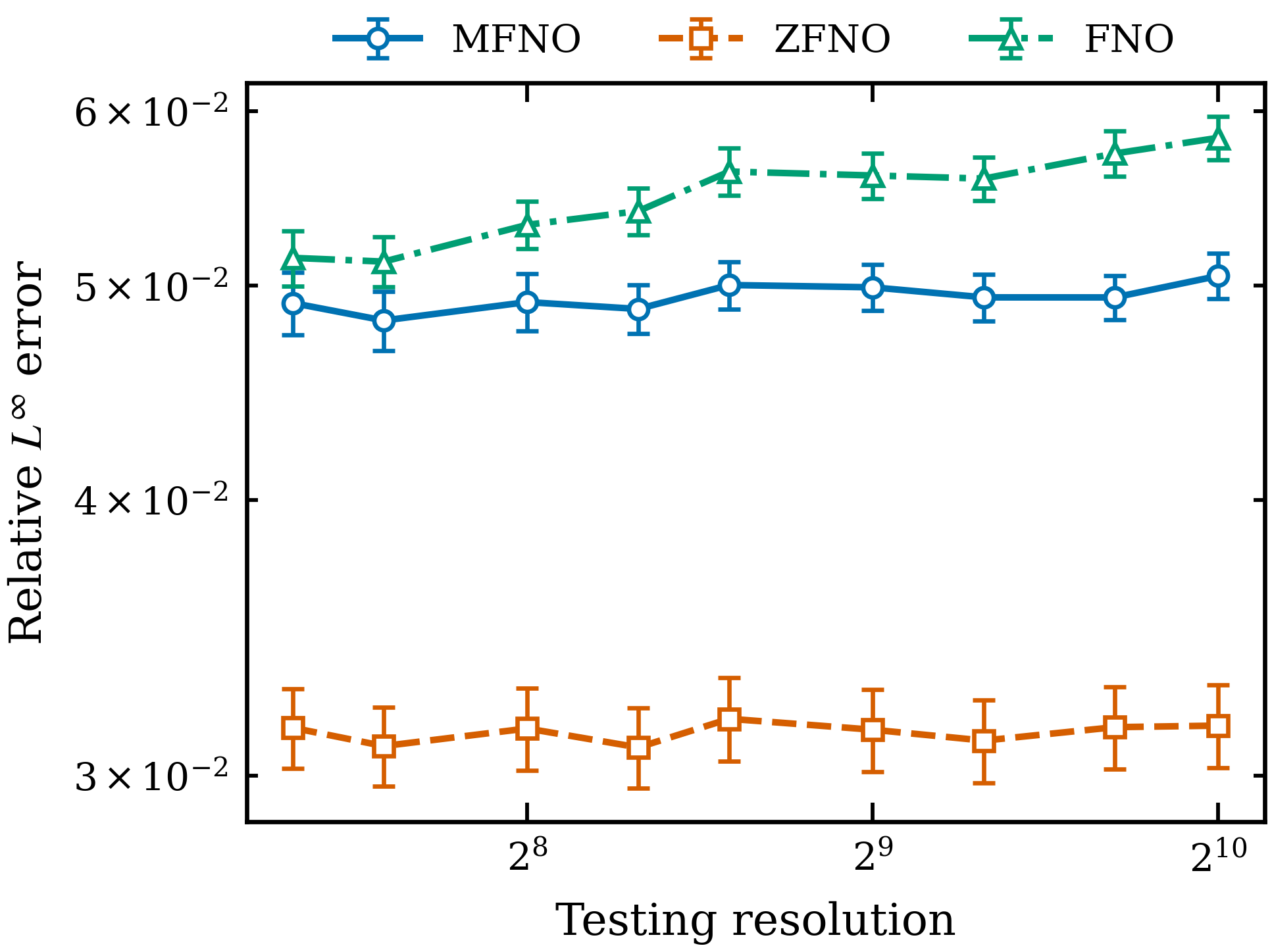}
    \caption{\TaskOne, $\RelLinf$}
\end{subfigure}
\begin{subfigure}{0.48\textwidth}
    \includegraphics[width=\linewidth]{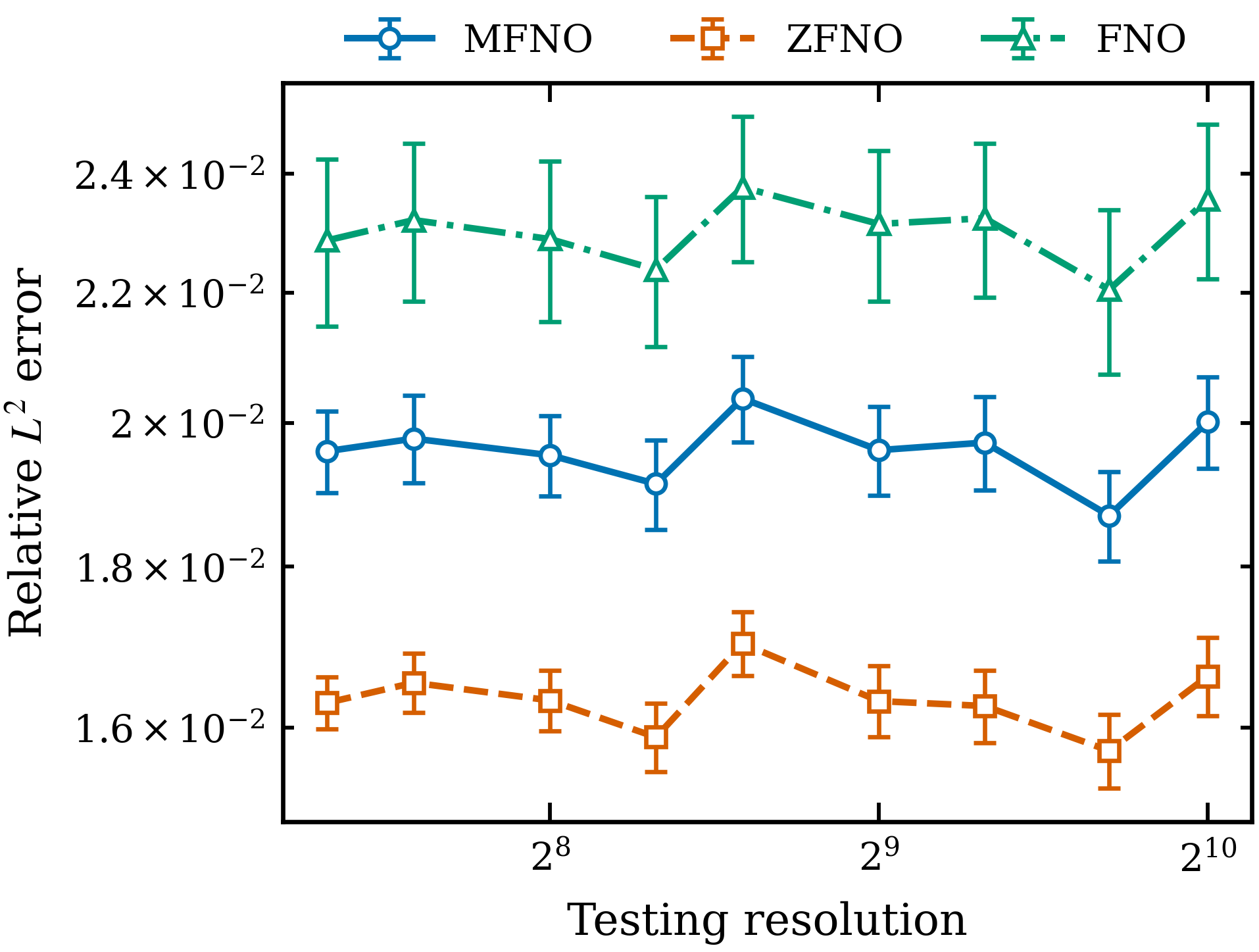}
    \caption{\TaskTwo, $\RelLtwo$}
\end{subfigure}
\begin{subfigure}{0.48\textwidth}
    \includegraphics[width=\linewidth]{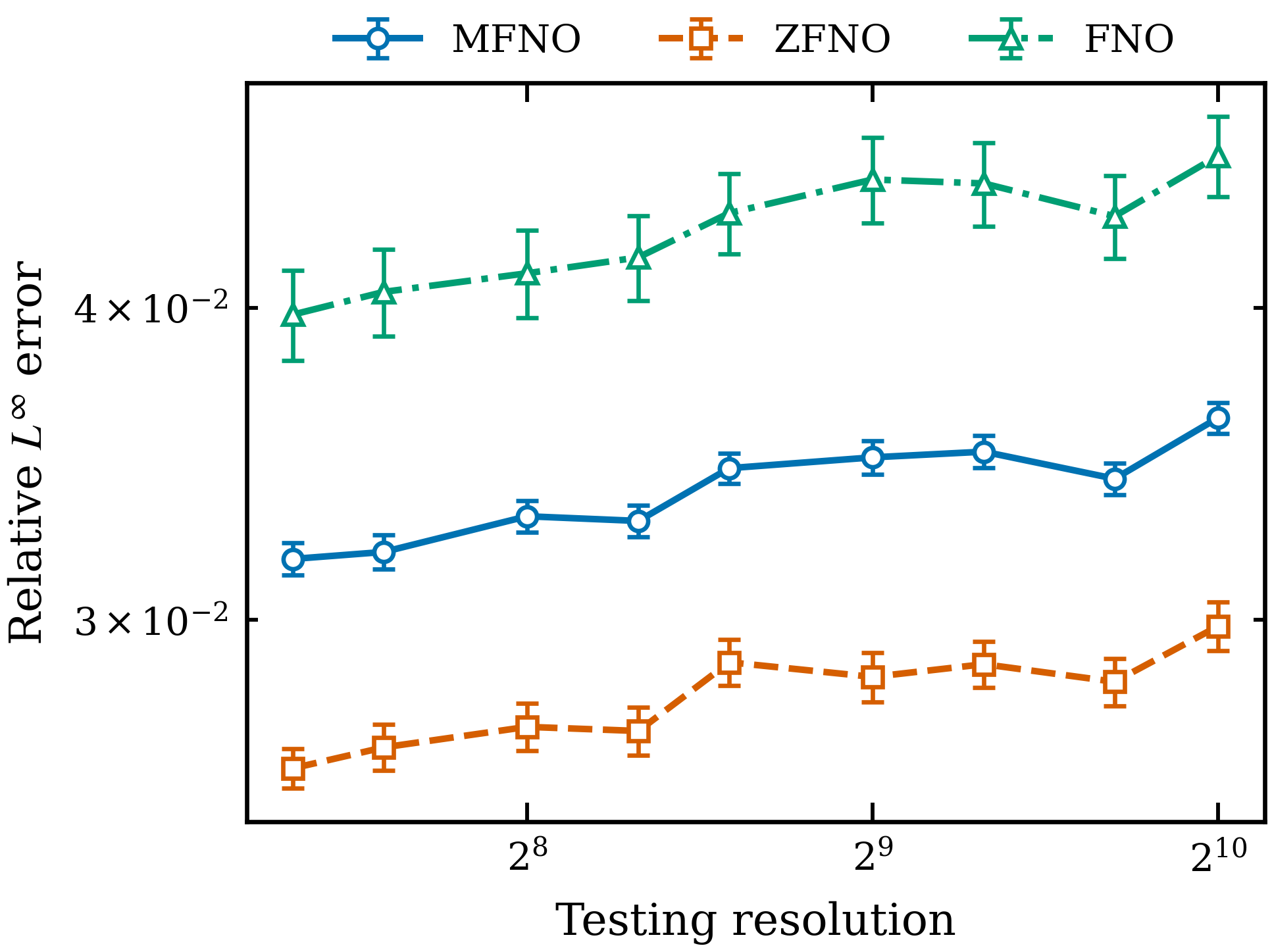}
    \caption{\TaskTwo, $\RelLinf$}
\end{subfigure}
\begin{subfigure}{0.48\textwidth}
    \includegraphics[width=\linewidth]{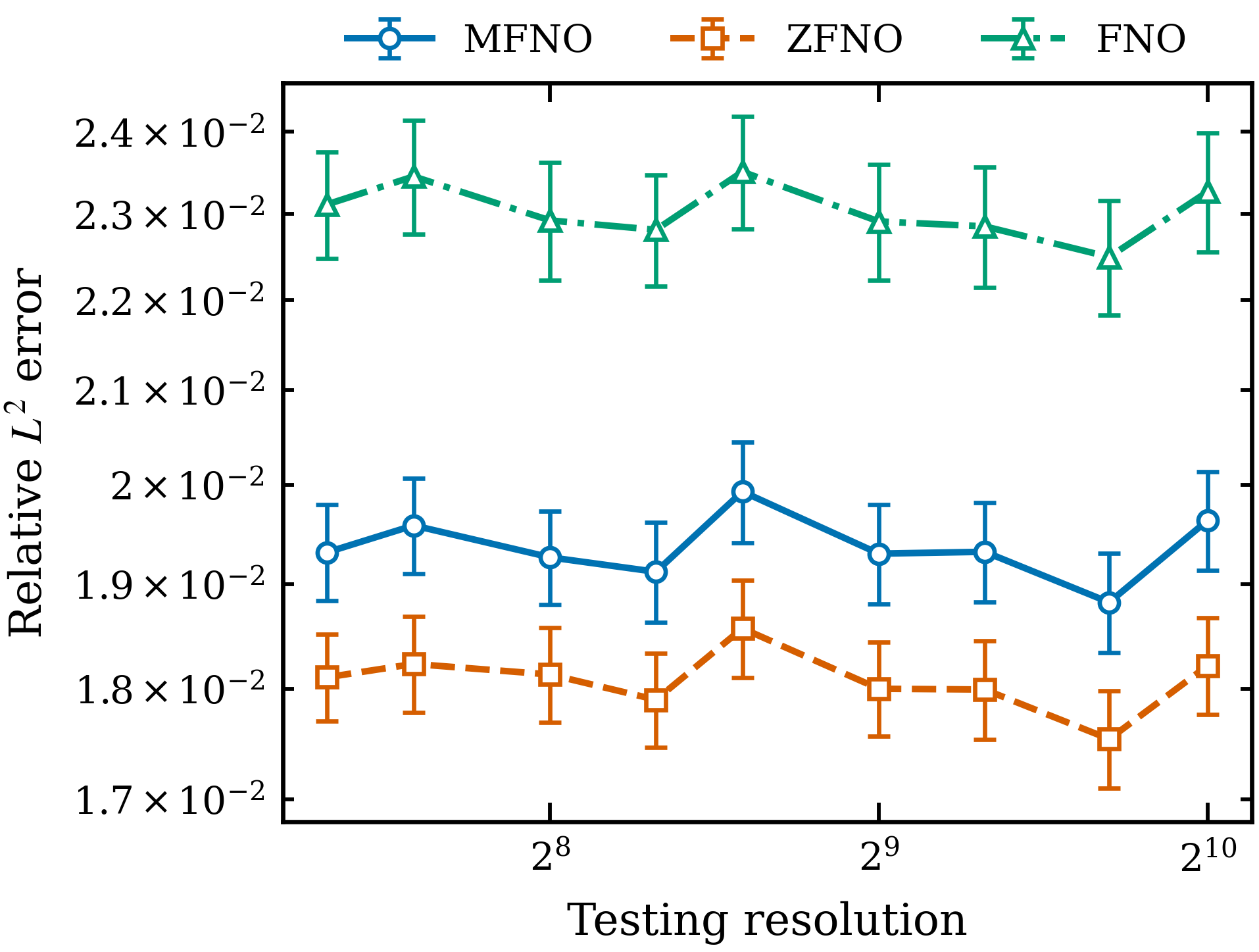}
    \caption{\TaskThree, $\RelLtwo$}
\end{subfigure}
\begin{subfigure}{0.48\textwidth}
    \includegraphics[width=\linewidth]{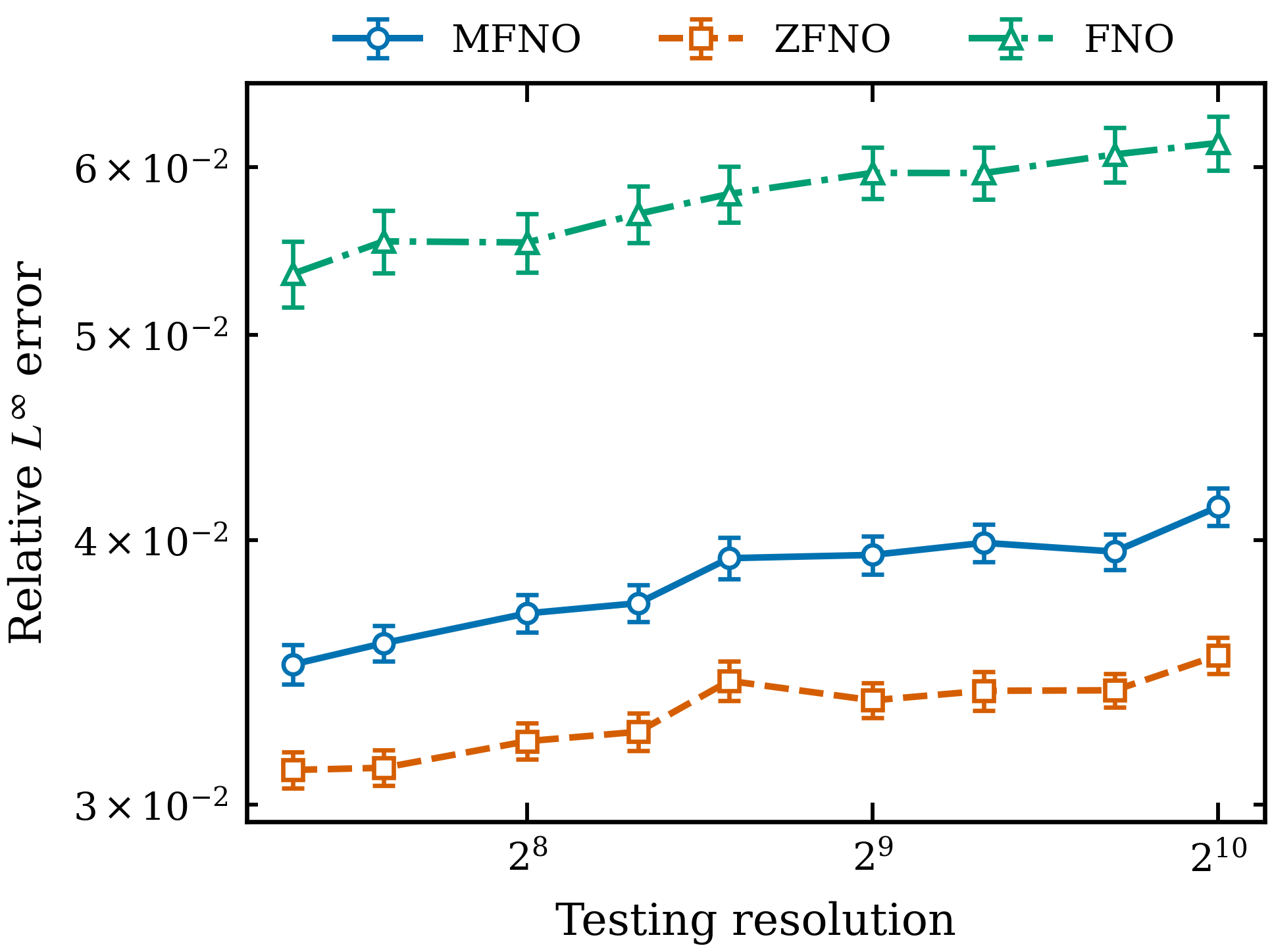}
    \caption{\TaskThree, $\RelLinf$}
\end{subfigure}
\caption{Resolution generalization for the FNO-family models. Error bars show
sample standard deviation over seeds.}
\label{fig:resolution-plots}
\end{figure}

\subsection{Parameter Counts and Inference Time}
Table~\ref{tab:timing} reports parameter counts and average inference time for
\TaskOne\ at resolution $128$. Timing was performed on \texttt{cuda} with an NVIDIA
GeForce RTX 4080, batch size $256$, $2048$ test paths per repeat, $2$ warmup
passes, and $5$ timed repeats per seed. Values are averaged over all 10 completed
seed checkpoints.

\begin{table}[htbp]
\centering
\caption{\TaskOne\ parameter counts and inference timing.}
\label{tab:timing}
\begin{tabular}{lrrrrr}
\toprule
Model & Parameters & Avg. time (s) & Std. time (s) & ms/path & Seeds \\
\midrule
MFNO & \num{1606337} & \num{2.866e-2} & \num{4.373e-3} & \num{1.399e-2} & 10 \\
ZFNO & \num{1606337} & \num{2.625e-2} & \num{1.208e-4} & \num{1.282e-2} & 10 \\
FNO & \num{1606337} & \num{1.926e-2} & \num{8.682e-4} & \num{9.402e-3} & 10 \\
LSTM & \num{67713} & \num{9.902e-3} & \num{2.834e-3} & \num{4.835e-3} & 10 \\
TCN & \num{5784604} & \num{1.879e-1} & \num{1.606e-4} & \num{9.173e-2} & 10 \\
DeepONet & \num{160344} & \num{8.289e-3} & \num{2.719e-3} & \num{4.047e-3} & 10 \\
\bottomrule
\end{tabular}
\end{table}

The TCN has the largest parameter count and is the slowest model in this timing
test. DeepONet is fastest on this hardware, while the FNO-family models have
identical parameter counts but different inference times due to their different
operator constructions.

\section{Conclusion}
\label{sec:con}

This paper extends the operator-learning framework of \citet{LKP:26} to stochastic Volterra integral equations. The Wong--Zakai convergence theorem, together with well-posedness and continuity of the associated deterministic solution maps, provides a foundation for approximating stochastic Volterra solution operators by mirror-padded Fourier neural operators. This extension requires separate convergence and approximation arguments that account for convolutional memory and are not covered by the preceding study.

Numerical experiments demonstrate competitive prediction accuracy across the Volterra examples considered here. The trained Fourier-based models retain comparable accuracy on finer temporal grids without retraining, supporting their use across different resolutions. Comparisons among architectures and padding strategies also highlight the influence of these choices on prediction accuracy and computational cost. These findings support the applicability of Fourier neural operators to stochastic systems with memory.

Future research may extend the analysis to singular memory kernels, investigate performance for more general nonlinear systems, and develop quantitative error bounds for learned solution operators.

\section*{Declaration of AI Assistance}
The authors used ChatGPT (OpenAI) to improve the language and readability of the article and to assist in correcting errors.
The ideas and results presented in this article are the authors' original intellectual contributions.
The authors take full responsibility for the content of the article, including its mathematical statements, proofs, and references.

\section*{Funding}
Hyungbin Park was supported by National Research Foundation of Korea grants funded by the Ministry of Science and ICT (2022R1A5A6000840, RS-2026-25488333).
Financial support from the Institute for Research in Finance and Economics of Seoul National University is gratefully acknowledged.  

 \section*{Disclosure Statement}
No potential conflict of interest was reported by the authors.

\section*{Data Availability Statement}
The data are available upon request.

\bibliographystyle{abbrvnat}
   \bibliography{wongzakai}

\clearpage
\appendix
\section{Architecture Search Grid}
\label{app:architecture-grid}

The same architecture candidate list was used for each of the three tasks.
Each entry below is one jointly specified architecture; the values in a tuple
were not independently combined. Each of the five architecture candidates was
crossed with the three learning rates
$\{10^{-3},5\mathord{\cdot}10^{-4},2\mathord{\cdot}10^{-4}\}$, resulting in
$15$ Stage-1 fits per task-model pair.

For MFNO, ZFNO, and FNO, an architecture is written as $(m,w,d,k)$, where
$m$ is the number of Fourier modes, $w$ is the channel width, $d$ is the
number of Fourier layers, and $k$ is the pointwise convolution kernel size.

{\scriptsize
\begin{table}[htbp]
\centering
\caption{Architecture candidates for the FNO-family models.}
\label{tab:fno-architecture-grid}
\resizebox{\textwidth}{!}{%
\begin{tabular}{lccccc}
\toprule
Model & Candidate 1 & Candidate 2 & Candidate 3 & Candidate 4 & Candidate 5 \\
\midrule
MFNO
& $(16,16,4,1)$
& $(32,32,4,1)$
& $(32,32,5,1)$
& $(32,64,5,1)$
& $(64,64,6,1)$ \\
ZFNO
& $(16,16,4,1)$
& $(32,32,4,1)$
& $(32,32,6,1)$
& $(32,64,6,1)$
& $(64,64,6,1)$ \\
FNO
& $(16,16,4,1)$
& $(32,32,4,1)$
& $(32,32,6,1)$
& $(32,64,6,1)$
& $(64,64,6,1)$ \\
\bottomrule
\end{tabular}%
}
\end{table}
}

For TCN, tuples are $(w,B,k,q)$, where $w$ is hidden width, $B$ is the
number of temporal convolution blocks, $k$ is kernel size, and $q$ is dropout.
For LSTM, tuples are $(h,L,q)$, where $h$ is hidden dimension, $L$ is the
number of recurrent layers, and $q$ is dropout. For DeepONet, tuples are
$(p,d_b,d_t,w)$, where $p$ is the branch--trunk product width, $d_b$ and
$d_t$ are the branch and trunk depths, and $w$ is their hidden-layer width.

{\scriptsize
\begin{table}[htbp]
\centering
\caption{Architecture candidates for TCN, LSTM, and DeepONet.}
\label{tab:baseline-architecture-grid}
\resizebox{\textwidth}{!}{%
\begin{tabular}{lccccc}
\toprule
Model & Candidate 1 & Candidate 2 & Candidate 3 & Candidate 4 & Candidate 5 \\
\midrule
TCN
& $(128,3,2,0.1)$
& $(128,6,3,0.2)$
& $(256,4,2,0.2)$
& $(256,6,2,0.2)$
& $(512,6,2,0.2)$ \\
LSTM
& $(128,1,0.0)$
& $(128,2,0.0)$
& $(256,2,0.1)$
& $(512,2,0.2)$
& $(512,3,0.2)$ \\
DeepONet
& $(128,2,2,128)$
& $(128,3,2,128)$
& $(300,3,2,128)$
& $(384,4,3,128)$
& $(384,4,3,256)$ \\
\bottomrule
\end{tabular}%
}
\end{table}
}

Thus, the complete Stage-1 sweep contains
$3\ \text{tasks}\times6\ \text{models}\times15=270$ fits. Stage~2 evaluates
five scheduler candidates for each selected task-model architecture, adding
$3\times6\times5=90$ fits. The full two-stage search therefore comprises
$360$ fitted models.
\section{Full Resolution-Generalization Error Table}
\label{app:resolution-table}
The following table gives the mean and sample standard deviation over seeds for
the higher-resolution FNO-family experiment. The training resolution $128$ is
omitted here by design.

\begin{landscape}
{\scriptsize
\begin{longtable}{llrcc}
\caption{Full resolution-generalization errors for MFNO, ZFNO, and FNO.}\\
\toprule
Task & Model & Resolution & $\RelLtwo$ mean $\pm$ std & $\RelLinf$ mean $\pm$ std \\
\midrule
\endfirsthead
\toprule
Task & Model & Resolution & $\RelLtwo$ mean $\pm$ std & $\RelLinf$ mean $\pm$ std \\
\midrule
\endhead
\TaskOne & FNO & 160 & \num{1.729e-02} $\pm$ \num{1.215e-03} & \num{5.148e-02} $\pm$ \num{1.483e-03} \\
\TaskOne & FNO & 192 & \num{1.678e-02} $\pm$ \num{1.242e-03} & \num{5.128e-02} $\pm$ \num{1.337e-03} \\
\TaskOne & FNO & 256 & \num{1.702e-02} $\pm$ \num{1.242e-03} & \num{5.328e-02} $\pm$ \num{1.327e-03} \\
\TaskOne & FNO & 320 & \num{1.674e-02} $\pm$ \num{1.309e-03} & \num{5.405e-02} $\pm$ \num{1.321e-03} \\
\TaskOne & FNO & 384 & \num{1.773e-02} $\pm$ \num{1.334e-03} & \num{5.633e-02} $\pm$ \num{1.400e-03} \\
\TaskOne & FNO & 512 & \num{1.690e-02} $\pm$ \num{1.321e-03} & \num{5.610e-02} $\pm$ \num{1.332e-03} \\
\TaskOne & FNO & 640 & \num{1.681e-02} $\pm$ \num{1.375e-03} & \num{5.592e-02} $\pm$ \num{1.261e-03} \\
\TaskOne & FNO & 832 & \num{1.719e-02} $\pm$ \num{1.376e-03} & \num{5.740e-02} $\pm$ \num{1.364e-03} \\
\TaskOne & FNO & 1024 & \num{1.736e-02} $\pm$ \num{1.405e-03} & \num{5.835e-02} $\pm$ \num{1.323e-03} \\
\TaskOne & MFNO & 160 & \num{1.419e-02} $\pm$ \num{4.427e-04} & \num{4.911e-02} $\pm$ \num{1.590e-03} \\
\TaskOne & MFNO & 192 & \num{1.372e-02} $\pm$ \num{4.552e-04} & \num{4.823e-02} $\pm$ \num{1.490e-03} \\
\TaskOne & MFNO & 256 & \num{1.394e-02} $\pm$ \num{4.884e-04} & \num{4.916e-02} $\pm$ \num{1.473e-03} \\
\TaskOne & MFNO & 320 & \num{1.376e-02} $\pm$ \num{4.995e-04} & \num{4.881e-02} $\pm$ \num{1.239e-03} \\
\TaskOne & MFNO & 384 & \num{1.458e-02} $\pm$ \num{5.082e-04} & \num{5.004e-02} $\pm$ \num{1.240e-03} \\
\TaskOne & MFNO & 512 & \num{1.394e-02} $\pm$ \num{4.938e-04} & \num{4.992e-02} $\pm$ \num{1.203e-03} \\
\TaskOne & MFNO & 640 & \num{1.386e-02} $\pm$ \num{5.036e-04} & \num{4.940e-02} $\pm$ \num{1.199e-03} \\
\TaskOne & MFNO & 832 & \num{1.403e-02} $\pm$ \num{5.167e-04} & \num{4.940e-02} $\pm$ \num{1.136e-03} \\
\TaskOne & MFNO & 1024 & \num{1.423e-02} $\pm$ \num{5.256e-04} & \num{5.052e-02} $\pm$ \num{1.196e-03} \\
\TaskOne & ZFNO & 160 & \num{1.001e-02} $\pm$ \num{5.464e-04} & \num{3.153e-02} $\pm$ \num{1.310e-03} \\
\TaskOne & ZFNO & 192 & \num{9.643e-03} $\pm$ \num{5.468e-04} & \num{3.094e-02} $\pm$ \num{1.272e-03} \\
\TaskOne & ZFNO & 256 & \num{9.587e-03} $\pm$ \num{5.550e-04} & \num{3.151e-02} $\pm$ \num{1.347e-03} \\
\TaskOne & ZFNO & 320 & \num{9.249e-03} $\pm$ \num{5.654e-04} & \num{3.089e-02} $\pm$ \num{1.291e-03} \\
\TaskOne & ZFNO & 384 & \num{9.881e-03} $\pm$ \num{6.210e-04} & \num{3.183e-02} $\pm$ \num{1.386e-03} \\
\TaskOne & ZFNO & 512 & \num{9.275e-03} $\pm$ \num{5.822e-04} & \num{3.147e-02} $\pm$ \num{1.359e-03} \\
\TaskOne & ZFNO & 640 & \num{9.287e-03} $\pm$ \num{6.064e-04} & \num{3.111e-02} $\pm$ \num{1.334e-03} \\
\TaskOne & ZFNO & 832 & \num{9.410e-03} $\pm$ \num{6.112e-04} & \num{3.155e-02} $\pm$ \num{1.348e-03} \\
\TaskOne & ZFNO & 1024 & \num{9.422e-03} $\pm$ \num{6.020e-04} & \num{3.160e-02} $\pm$ \num{1.368e-03} \\
\TaskTwo & FNO & 160 & \num{2.286e-02} $\pm$ \num{1.398e-03} & \num{3.974e-02} $\pm$ \num{1.650e-03} \\
\TaskTwo & FNO & 192 & \num{2.320e-02} $\pm$ \num{1.338e-03} & \num{4.058e-02} $\pm$ \num{1.626e-03} \\
\TaskTwo & FNO & 256 & \num{2.288e-02} $\pm$ \num{1.347e-03} & \num{4.128e-02} $\pm$ \num{1.668e-03} \\
\TaskTwo & FNO & 320 & \num{2.237e-02} $\pm$ \num{1.230e-03} & \num{4.189e-02} $\pm$ \num{1.644e-03} \\
\TaskTwo & FNO & 384 & \num{2.376e-02} $\pm$ \num{1.265e-03} & \num{4.364e-02} $\pm$ \num{1.620e-03} \\
\TaskTwo & FNO & 512 & \num{2.313e-02} $\pm$ \num{1.274e-03} & \num{4.501e-02} $\pm$ \num{1.767e-03} \\
\TaskTwo & FNO & 640 & \num{2.323e-02} $\pm$ \num{1.311e-03} & \num{4.485e-02} $\pm$ \num{1.725e-03} \\
\TaskTwo & FNO & 832 & \num{2.205e-02} $\pm$ \num{1.326e-03} & \num{4.351e-02} $\pm$ \num{1.657e-03} \\
\TaskTwo & FNO & 1024 & \num{2.356e-02} $\pm$ \num{1.330e-03} & \num{4.600e-02} $\pm$ \num{1.700e-03} \\
\TaskTwo & MFNO & 160 & \num{1.959e-02} $\pm$ \num{5.868e-04} & \num{3.173e-02} $\pm$ \num{4.716e-04} \\
\TaskTwo & MFNO & 192 & \num{1.977e-02} $\pm$ \num{6.305e-04} & \num{3.194e-02} $\pm$ \num{5.053e-04} \\
\TaskTwo & MFNO & 256 & \num{1.953e-02} $\pm$ \num{5.763e-04} & \num{3.301e-02} $\pm$ \num{4.793e-04} \\
\TaskTwo & MFNO & 320 & \num{1.913e-02} $\pm$ \num{6.256e-04} & \num{3.286e-02} $\pm$ \num{4.795e-04} \\
\TaskTwo & MFNO & 384 & \num{2.036e-02} $\pm$ \num{6.381e-04} & \num{3.450e-02} $\pm$ \num{4.842e-04} \\
\TaskTwo & MFNO & 512 & \num{1.960e-02} $\pm$ \num{6.372e-04} & \num{3.485e-02} $\pm$ \num{5.356e-04} \\
\TaskTwo & MFNO & 640 & \num{1.971e-02} $\pm$ \num{6.706e-04} & \num{3.503e-02} $\pm$ \num{5.178e-04} \\
\TaskTwo & MFNO & 832 & \num{1.869e-02} $\pm$ \num{6.119e-04} & \num{3.417e-02} $\pm$ \num{4.952e-04} \\
\TaskTwo & MFNO & 1024 & \num{2.002e-02} $\pm$ \num{6.719e-04} & \num{3.614e-02} $\pm$ \num{5.122e-04} \\
\TaskTwo & ZFNO & 160 & \num{1.629e-02} $\pm$ \num{3.082e-04} & \num{2.617e-02} $\pm$ \num{4.791e-04} \\
\TaskTwo & ZFNO & 192 & \num{1.654e-02} $\pm$ \num{3.570e-04} & \num{2.668e-02} $\pm$ \num{5.666e-04} \\
\TaskTwo & ZFNO & 256 & \num{1.632e-02} $\pm$ \num{3.644e-04} & \num{2.719e-02} $\pm$ \num{5.947e-04} \\
\TaskTwo & ZFNO & 320 & \num{1.589e-02} $\pm$ \num{3.994e-04} & \num{2.708e-02} $\pm$ \num{5.981e-04} \\
\TaskTwo & ZFNO & 384 & \num{1.702e-02} $\pm$ \num{3.969e-04} & \num{2.885e-02} $\pm$ \num{6.137e-04} \\
\TaskTwo & ZFNO & 512 & \num{1.631e-02} $\pm$ \num{4.248e-04} & \num{2.847e-02} $\pm$ \num{6.475e-04} \\
\TaskTwo & ZFNO & 640 & \num{1.626e-02} $\pm$ \num{4.294e-04} & \num{2.880e-02} $\pm$ \num{6.123e-04} \\
\TaskTwo & ZFNO & 832 & \num{1.573e-02} $\pm$ \num{4.268e-04} & \num{2.834e-02} $\pm$ \num{6.213e-04} \\
\TaskTwo & ZFNO & 1024 & \num{1.661e-02} $\pm$ \num{4.765e-04} & \num{2.983e-02} $\pm$ \num{6.684e-04} \\
\TaskThree & FNO & 160 & \num{2.312e-02} $\pm$ \num{6.353e-04} & \num{5.345e-02} $\pm$ \num{1.923e-03} \\
\TaskThree & FNO & 192 & \num{2.345e-02} $\pm$ \num{6.880e-04} & \num{5.537e-02} $\pm$ \num{1.880e-03} \\
\TaskThree & FNO & 256 & \num{2.292e-02} $\pm$ \num{6.957e-04} & \num{5.530e-02} $\pm$ \num{1.763e-03} \\
\TaskThree & FNO & 320 & \num{2.281e-02} $\pm$ \num{6.578e-04} & \num{5.705e-02} $\pm$ \num{1.752e-03} \\
\TaskThree & FNO & 384 & \num{2.351e-02} $\pm$ \num{6.826e-04} & \num{5.829e-02} $\pm$ \num{1.757e-03} \\
\TaskThree & FNO & 512 & \num{2.291e-02} $\pm$ \num{6.873e-04} & \num{5.964e-02} $\pm$ \num{1.668e-03} \\
\TaskThree & FNO & 640 & \num{2.285e-02} $\pm$ \num{7.086e-04} & \num{5.963e-02} $\pm$ \num{1.689e-03} \\
\TaskThree & FNO & 832 & \num{2.250e-02} $\pm$ \num{6.661e-04} & \num{6.085e-02} $\pm$ \num{1.804e-03} \\
\TaskThree & FNO & 1024 & \num{2.327e-02} $\pm$ \num{7.130e-04} & \num{6.162e-02} $\pm$ \num{1.808e-03} \\
\TaskThree & MFNO & 160 & \num{1.931e-02} $\pm$ \num{4.796e-04} & \num{3.495e-02} $\pm$ \num{7.484e-04} \\
\TaskThree & MFNO & 192 & \num{1.958e-02} $\pm$ \num{4.796e-04} & \num{3.577e-02} $\pm$ \num{6.879e-04} \\
\TaskThree & MFNO & 256 & \num{1.926e-02} $\pm$ \num{4.622e-04} & \num{3.694e-02} $\pm$ \num{7.459e-04} \\
\TaskThree & MFNO & 320 & \num{1.912e-02} $\pm$ \num{4.951e-04} & \num{3.736e-02} $\pm$ \num{7.427e-04} \\
\TaskThree & MFNO & 384 & \num{1.993e-02} $\pm$ \num{5.186e-04} & \num{3.923e-02} $\pm$ \num{8.819e-04} \\
\TaskThree & MFNO & 512 & \num{1.930e-02} $\pm$ \num{4.967e-04} & \num{3.937e-02} $\pm$ \num{8.234e-04} \\
\TaskThree & MFNO & 640 & \num{1.932e-02} $\pm$ \num{4.933e-04} & \num{3.989e-02} $\pm$ \num{8.155e-04} \\
\TaskThree & MFNO & 832 & \num{1.882e-02} $\pm$ \num{4.831e-04} & \num{3.951e-02} $\pm$ \num{7.665e-04} \\
\TaskThree & MFNO & 1024 & \num{1.963e-02} $\pm$ \num{5.016e-04} & \num{4.148e-02} $\pm$ \num{8.487e-04} \\
\TaskThree & ZFNO & 160 & \num{1.811e-02} $\pm$ \num{4.048e-04} & \num{3.116e-02} $\pm$ \num{6.168e-04} \\
\TaskThree & ZFNO & 192 & \num{1.823e-02} $\pm$ \num{4.520e-04} & \num{3.124e-02} $\pm$ \num{5.984e-04} \\
\TaskThree & ZFNO & 256 & \num{1.813e-02} $\pm$ \num{4.444e-04} & \num{3.215e-02} $\pm$ \num{6.296e-04} \\
\TaskThree & ZFNO & 320 & \num{1.790e-02} $\pm$ \num{4.375e-04} & \num{3.248e-02} $\pm$ \num{6.614e-04} \\
\TaskThree & ZFNO & 384 & \num{1.857e-02} $\pm$ \num{4.640e-04} & \num{3.433e-02} $\pm$ \num{7.309e-04} \\
\TaskThree & ZFNO & 512 & \num{1.800e-02} $\pm$ \num{4.373e-04} & \num{3.361e-02} $\pm$ \num{6.421e-04} \\
\TaskThree & ZFNO & 640 & \num{1.799e-02} $\pm$ \num{4.610e-04} & \num{3.396e-02} $\pm$ \num{7.122e-04} \\
\TaskThree & ZFNO & 832 & \num{1.754e-02} $\pm$ \num{4.410e-04} & \num{3.398e-02} $\pm$ \num{6.290e-04} \\
\TaskThree & ZFNO & 1024 & \num{1.822e-02} $\pm$ \num{4.533e-04} & \num{3.530e-02} $\pm$ \num{6.877e-04} \\
\bottomrule
\end{longtable}
}
\end{landscape}

\end{document}